\documentclass[journal]{IEEEtran}

\usepackage{hyperref}
\usepackage{url}
\usepackage{graphicx,subfigure,color}
\usepackage{amsthm}
\usepackage{amssymb}
\usepackage{amsmath,bm}
\theoremstyle{definition}

\usepackage{algorithmic,algorithm}
\newtheorem{theorem}{Theorem}
\newtheorem{lemma}[theorem]{Lemma}
\usepackage{cite}

\begin{document}

\title{OSLNet: Deep Small-Sample Classification with an Orthogonal Softmax Layer}

\author{Xiaoxu Li, Dongliang Chang, Zhanyu Ma, Zheng-Hua Tan, Jing-Hao Xue, Jie Cao, Jingyi Yu, and Jun Guo%
\thanks{X. Li and J. Cao are with School of Computer and Communication, Lanzhou University of Technology, China.}
\thanks{X. Li, D. Chang, Z. Ma and J. Guo are with the Pattern Recognition and Intelligent System Laboratory, School of Artificial Intelligence, Beijing University of Posts and Telecommunications, Beijing, China.}
\thanks{Z.-H. Tan is with the Department of Electronic Systems, Aalborg University, Denmark.}
\thanks{J.-H. Xue is with the Department of Statistical Science, University College London, U.K.}
\thanks{J. Yu is with the School of Information Science and Technology, ShanghaiTech University, China.}
}

\markboth{}{Li \MakeLowercase{\textit{et al.}}: OSLNet}

\maketitle

\begin{abstract}
A deep neural network of multiple nonlinear layers forms a large function space, which can easily lead to overfitting when it encounters small-sample data. To mitigate overfitting in small-sample classification, learning more discriminative features from small-sample data is becoming a new trend. To this end, this paper aims to find a subspace of neural networks that can facilitate a large decision margin. Specifically, we propose the~\emph{Orthogonal Softmax Layer} (OSL), which makes the weight vectors in the classification layer remain orthogonal during both the training and test processes. The Rademacher complexity of a network using the OSL is only $\frac{1}{K}$, where $K$ is the number of classes, of that of a network using the fully connected classification layer, leading to a tighter generalization error bound. Experimental results demonstrate that the proposed OSL has better performance than the methods used for comparison on four small-sample benchmark datasets, as well as its applicability to large-sample datasets. Codes are available at:~\url{https://github.com/dongliangchang/OSLNet}.
\end{abstract}

\begin{IEEEkeywords}
Deep neural network, Orthogonal softmax layer, Overfitting, Small-sample classification.
\end{IEEEkeywords}

\IEEEpeerreviewmaketitle

\section{Introduction}
In recent years, due to the advances of deep learning, large-scale image classification has achieved massive success~\cite{lecun2015deep,szegedy2015going}. However, in many real-world problems, big data are difficult to obtain~\cite{bruzzone2010domain,dong2018few,snell2017prototypical}. In addition, when human beings learn a concept, millions or billions of data samples are unnecessary~\cite{fei2006one,vinyals2016matching}. Therefore, small-sample classification~\cite{shu2018small,zhang2019few,jiang2018learning} has received much attention recently in the deep learning community~\cite{fu2015transductive,santoro2016one,vinyals2016matching,lu2018embarrassingly,choi2018structured,fu2018zero}. However, a deep neural network usually models a large function space, which results in overfitting and instability problems that are difficult to avoid in small-sample classification~\cite{glorot2010understanding,zhang2016understanding,courty2017optimal}. 

\begin{figure*}[t] 
\begin{center}
\includegraphics[width=7.2in]{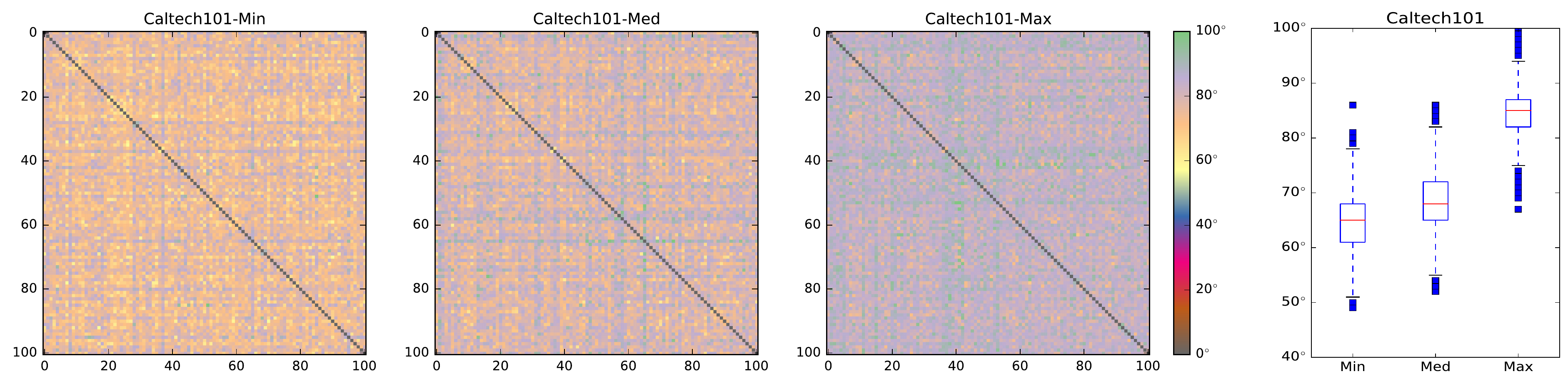}
\end{center}
\caption{The first three matrices show the final angles of the weight vectors from the classification layer in the fully connected network (FC). We ran $60$ rounds of simulations on the Caltech101 dataset, see Section~\ref{sec:Experiments} for more details. The results are selected from the minimum (Caltech101-Min), median (Caltech101-Med) and maximum (Caltech101-Max) of $60$ sets of accuracies. The corresponding accuracies are $ 88.07\% $, $89.28\% $ and $90.35\% $, respectively. The boxplot shows the off-diagonal angels of these three matrices.} \label{angel}
\end{figure*} 

Small-sample classification can be roughly classified into two families depending on whether there are unseen categories to be predicted\cite{chen2019Close,sun2019meta,wei2019piecewise}. In this work, we mainly focus on the one with no unseen category to be predicted. A key challenge in this direction is how to avoid overfitting. Up to now, many methods have been proposed to mitigate overfitting in small-sample classification, such as data augmentation~\cite{kim2017feature,perez2017effectiveness,zhong2018camstyle}, domain adaptation~\cite{bruzzone2010domain,courty2017optimal}, regularization~\cite{srivastava2014dropout}, ensemble methods~\cite{huang2017SnapShot} and learning discriminative features~\cite{liu2016large}. 
 
Recently, learning discriminative features has become a new trend to improve the classification performance in deep learning. Methods such as the large-margin loss and the virtual softmax method~\cite{chen2018virtual} work well on both large-sample and small-sample data. However, these methods either add some constraints on the loss function or make some assumptions on data, which increases the difficulty of optimization and limits the applicable types of data. On the contrary, our goal in this work is to find a subspace of neural networks that can readily obtain a large decision margin and learn highly discriminative features. Specifically, we aim at obtaining a large decision margin through achieving large angles between the weight vectors of different classes in the classification layer (i.e., output layer), in part motivated by the observation that the larger the angles between the weight vectors in the classification layer are, the better the generalization performance is, as shown in Fig.~\ref{angel}. 

Therefore, we propose the \emph{Orthogonal Softmax Layer} (OSL) for neural networks as a replacement of the fully connected classification layer. In the proposed OSL, some connections are removed and the weight vectors of different classes are pairwise orthogonal. Due to fewer connections and larger between-class angles, the OSL can mitigate the co-adaptation~\cite{hinton2012improving} of a network while enhancing the discrimination ability of features, as we will show in this work. Compared with traditional networks with a fully connected classification layer, a neural network with the OSL has significantly lower model complexity and is ideally suitable for small-sample classification. Experimental results demonstrate that the proposed OSL performs better on four small-sample benchmark datasets than the methods used for comparison as well as its applicability to large-sample datasets.

A number of methods have been proposed to maintain the orthogonality of the weight vectors during the training process of networks~\cite{arjovsky2016unitary,huang2018orthogonal,sun2017svdnet,xie2017all,liu2017deep,liu2018learning,mettes2019hyperspherical}. These methods maintain the orthogonality of weight vectors either for reducing gradients vanishing and obtaining a stable feature distribution~\cite{arjovsky2016unitary,huang2018orthogonal,xie2017all} or for generating decorrelated feature representation~\cite{sun2017svdnet,rodriguez2016regularizing}. Unlike these methods, our work does not constrain the optimization process to obtain an orthogonal weight matrix, but constructs a network structure with a fixed orthogonal classification layer by removing some connections so that it can mitigate the co-adaptation between the parameters.
The main contributions of this paper are threefold: 
\begin{enumerate}
\item A novel layer, namely the Orthogonal Softmax Layer (OSL), is proposed. The OSL is an alternative to the fully connected classification layer and can be used as the classification layer (i.e., the output layer) of any neural network for classification.

\item The proposed OSL can reduce the difficulty of network optimization due to militating the co-adaptation between the parameters in the classification layer.

\item A network with the proposed OSL can have a lower generalization error bound than a network with a fully connected classification layer. 
\end{enumerate}

\section{Related Work}
\textbf{Data augmentation.} Data augmentation, which artificially inflates the training set with label preserving transformations, is well-suited to limited training data~\cite{cubuk2019autoaugment,volpi2018generalizing}, such as the methods of deformation~\cite{kulkarni2015deep,ratner2017learning}, generating more training samples~\cite{shrivastava2017learning}, and pseudo-label~\cite{ratner2016data}. However, data augmentation is computationally costly to implement~\cite{perez2017effectiveness}. 

\textbf{Domain adaptation.} The goal of domain adaptation~\cite{wang2018deep,yoon2019tapnet} is to use a model, which is trained in the source domain with a sufficient amount of annotated training data while trained in the target domain with little or no training data~\cite{yosinski2014transferable,elhoseiny2017write,tzeng2017adversarial,rozantsev2018residual}. The simplest approach to domain adaptation is to use the available annotated data in the target domain to fine-tune a convolutional neural network (CNN) pre-trained on the source data, such as ImageNet~\cite{pan2010survey,tzeng2017adversarial,oquab2014learning,long2017deep}, which is a commonly used method for small-sample classification~\cite{ouyang2016factors}. However, as both the initial learning rate and the optimization strategy of neural networks affect the final performance, this kind of methods is difficult to avoid the overfitting problem in small-sample classification. In addition, the knowledge distillation~\cite{hinton2015distilling}, a kind of method for knowledge transfer, compresses the knowledge in an ensemble into a single model. Overall, this type of methods have some limitations: the original domain and the target domain cannot be far away from each other, and overfitting of neural networks on small-sample data remains difficult to avoid. 

\textbf{Learning discriminative features.} There are several recent studies that explicitly encourage the learning of discriminative features and enlarge the decision margin, such as the virtual softmax~\cite{chen2018virtual}, the L-softmax loss~\cite{liu2016large}, the A-softmax loss~\cite{liu2017sphereface}, the GM loss~\cite{wan2018rethinking}, and the center loss~\cite{wen2016discriminative}. The virtual softmax enhances the discrimination ability of learned features by injecting a dynamic virtual negative class into the original softmax, and it indeed encourages the features to be more compact and separable across the classes~\cite{chen2018virtual}. The L-softmax loss and the A-softmax loss are built on the cross-entropy loss. They introduce new classification scores to enlarge the decision margin. However, the two losses increase the difficulty of network optimization. The GM loss assumes that the deep features of sample points follow a Gaussian mixture distribution. It still has some limitations since many real data are not well suitable to be modelled by Gaussian mixture distributions. The center loss~\cite{wen2016discriminative} developed a regularization term for the softmax loss function, that is, features of the samples from the same class must be close in the Euclidean distance. In addition to the  loss functions above, other loss functions either consider the imbalance of data~\cite{lin2017focal} or the noisy labels in the training data~\cite{zhang2018generalized}. The focal loss~\cite{lin2017focal} places different weights on different training samples: the samples that are difficult to be identified will be assigned a large weight. The truncated $L_{q}$ loss~\cite{zhang2018generalized}, a noise-robust loss function, can overcome noisy labels in the training data.
Unlike these studies for improving the loss function, our method obtains discriminative features by constructing a fixed orthogonal classification layer by removing some connections.

\textbf{Ensemble methods.} The ensemble methods~\cite{granitto2005neural,brown2003use,singh2016swapout,laine2016temporal,kumar2017ensemble} have been shown to be effective to address the overfitting problem. The SnapShot ensembling~\cite{huang2017SnapShot} only trains one time and obtains multiple base classifiers for free. This method leverages the nonconvex nature of neural networks and the ability of the stochastic gradient descent (SGD) to converge to or escape from local minima on demand. It finds multiple local minima of loss, saves the weights of the corresponding base networks, and combines the predictions of the corresponding base networks. The temporal ensembling, a parallel work to the SnapShot ensembling, can be trained on a single network. The predictions made on different epochs correspond to an ensemble prediction of multiple subnetworks due to dropout regularization~\cite{laine2016temporal}. Both the SnapShot ensembling and the temporal ensembling work well on small-sample data. However, generally speaking, using an entire ensemble of models to make prediction is cumbersome and computationally costly. 
 
\textbf{Regularization.} The $L_{2}$ regularization method~\cite{neumaier1998solving} is often applied to mitigate overfitting in neural networks. Dropout~\cite{wager2013dropout} mitigates overfitting mainly by randomly dropping some units (with their connections) from the neural network during training, which prevents the units from coadapting too much. DropConnect~\cite{wan2013regularization} is a variant of Dropout, where each connection can be dropped with probability $p$. Both approaches introduce dynamic sparsity within the model. The fully connected layer with DropConnect becomes a sparsely connected layer, where the connections are randomly selected during the training stage. In addition, the orthogonality regularization is also a kind of regularization techniques that are able to stabilize the distribution of activation over layers within CNNs and improve the performance of CNNs~\cite{bansal2018can,rodriguez2016regularizing}.

In addition to these methods, there
exist implicit regularization techniques, such as batch normalization and its variants~\cite{huang2018decorrelated,huang2019iterative}. Batch normalization (BN)~\cite{ioffe2015batch} aims to normalize the distribution of the layer input of neural networks, hence it can reduce the internal covariate shift in the training process. Built on BN, the decorrelated batch normalization (DBN)~\cite{huang2018decorrelated} decorrelates the layer input and improved BN on CNNs. Iterative normalization (IterNorm)~\cite{huang2019iterative} further improves DBN towards efficient whitening and iteratively normalizes the data along the eigenvectors in the training process. Both DBN and IterNorm adjust the distribution of samples so that features of samples are pairwise orthogonal. Unlike them, the proposed OSL forces the weights in the classification layer to be pairwise orthogonal (Please refer to Fig.~\ref{net} for details.) The orthogonality in OSL is to assign the input neurons to different output neurons so that it can mitigate the co-adaptation between the parameters.

\section{The Orthogonal Softmax Layer} 

Dropping some connections in the neural networks are adopted by Dropout and DropConnect, as well as the proposed OSL. To make the proposed OSL easy to understand, we first review Dropout and  DropConnect. 
\begin{figure*}[h]
\begin{center}
  \includegraphics[width=4.3in]{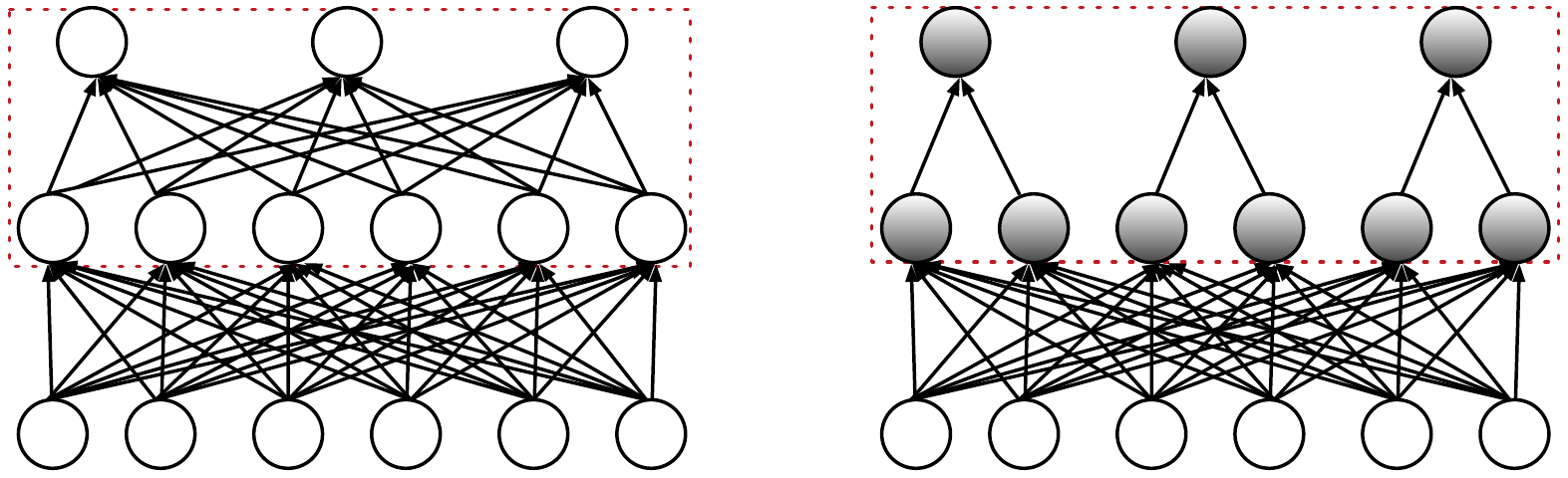}
\end{center}
\caption{The left-hand panel is a standard network with the fully connected classification layer, and the right-hand panel is a network with the Orthogonal Softmax Layer (OSL, indicated by shadow).}\label{net}
\end{figure*}

\subsection{Preliminaries}
We denote the input vector and the output vector of a layer in the neural network as $\mathbf{v} = [v_1,v_2,\ldots,v_D]^\mathrm{T}$ and $\mathbf{r }= [r_1,r_2,\ldots,r_K]^\mathrm{T}$, respectively, and denote the activation function as $a$. Following the formulation in~\cite{wan2013regularization}, the fully connected layer is defined as $\mathbf{r} = a (\mathbf{W} \mathbf{v})$ with weight matrix $\mathbf{W}$.

Based on the above denotations, the fully connected layer with Dropout is defined as $\mathbf{r} = \mathbf{m}\star a \left (\mathbf{W} \mathbf{v}  \right )$, where $\star$ represents element wise product, $\mathbf{m}$ is a $K$-dimension binary vector. The $j$th element of $\mathbf{m}$, $m_j \sim \text{Bernoulli }\left(1-p\right )$, $j\in \{1,2,\ldots,K\}$ and $p$ is probability for dropping a neuron. Moreover, the fully connected layer with DropConnect is defined as $\mathbf{r}= a \left (\left (\mathbf{M} \star \mathbf{W}  \right ) \mathbf{v}\right )$, where $\star$ represents element wise product and each entry of binary matrix $\mathbf{M}$ is $M_{ij}\sim \text{Bernoulli }\left(1-q\right )$, $i\in \{1,2,\ldots,D\}$ and $j\in \{1,2,\ldots,K\}$, and $q$ is the probability for dropping a connection.

\subsection{Mathematical Interpretation of the OSL}
To maintain large angles among the weights in the classification layer, we drop some connections in the fully connected classification layer, make the weights in the classification layer be pairwise orthogonal and propose the \emph{Orthogonal Softmax Layer} (\emph{OSL}, see Fig.~\ref{net}). The OSL is defined as 
\begin{equation}
\mathbf{r} = softmax \left (\left (\mathbf{M} \star  \mathbf{W}  \right )\mathbf{v} \right),
\label{eq:osl}
\end{equation} 
where $\star$ represents element wise product, and $\mathbf{M}$, the mask matrix of $\mathbf{W}$, is a fixed and predesigned block diagonal matrix as
\begin{equation}
\mathbf{M}= \begin{bmatrix}
&M_{1,1} &0_{1,2} &\cdots  &0_{1,K} \\ 
&0_{2,1}& M_{2,2} &\cdots &0_{2,K}\\ 
&\vdots   &\vdots   & \ddots &\vdots  \\ 
&0_{K,1} &0_{K,2} &\cdots   &M_{K,K}
\end{bmatrix},
\label{eq:M}
\end{equation}
where $K$ is the number of classes, $M_{ij}$ is a column vector whose elements are $1$, and $0_{ij}$ is a zero column vector where every element is $0$. The matrix is fixed during the training and test phases. If we consider $\mathbf{M} \star  \mathbf{W}$ as one matrix, the column vectors are pairwise orthogonal, which is equivalent to introducing a strong prior that the angles between the weights of different categories are all $90^\circ$.

\subsection{Remarks for OSL}
Here are some remarks on the OSL. First, the OSL is an alternative to the classification layer (i.e., the last fully connected layer) of a neural network. In contrast, the Dropconnect cannot be effectively used in the classification layer, because if all the connections of a output neuron are dropped, the neural network or part of it cannot be trained.

Second, a neural network with the OSL is a single model throughout the training and test processes, since it first drops some connections and subsequently fixes the structure in both the training and test phases. This is in contrast with the DropConnect method. DropConnect randomly drops connections with a given probability during the training phase, but no connection is dropped in the test phase. DropConnect can be considered an implicit ensemble method. 

Third, in the OSL, different neurons of the last hidden layer connect to different output neurons. This setup assigns each hidden neuron to one and only one specific class that they are responsible prior to the start of training, so the difficulty level of training the neural network is reduced. Moreover, the solution space of a network with the OSL is a subset of the solution space of its corresponding network with a fully connected classification layer, which implies low model complexity of the former.

\subsection{Implementation of the OSLNet} 
The OSL can be used in any type of networks for classification, such as the fully connected network or CNN. For convenience, we call a neural network with the OSL as \emph{\textbf{OSLNet}}. The OSLNet uses the OSL as the classification layer instead of a fully connected classification layer, and the structure of the other layers are kept the same as a standard network, e.g., a neural network with the fully connected classification layer. 

Different from the implementation of the fully connected classification layer, the forward computation of the OSL needs to conduct the dot product (see (\ref{eq:osl})) between the predesigned matrix $\mathbf{M}$ (see (\ref{eq:M})) and weights matrix $\mathbf{W}$.

\subsection{The Generalization Error Bound of the OSLNet}
In this section, we discuss the generalization error bound of the OSLNet in terms of the Rademacher complexity~\cite{mohri2012foundations,bartlett2002rademacher}. We define the entire network into two parts: the classification layer and the layers of extracting features. The classification layer refers to the last fully connected layer in a neural network, and the layers of extracting features refer to all layers except the classification layer. Based on the definitions, a standard network with the fully connected classification layer can be represented as $f \left ( x; \mathbf{W}_{s}, \mathbf{W}_{g} \right )$ and the OSLNet can be represented as $f \left ( x; \mathbf{M} \star \mathbf{W}_{s}, \mathbf{W}_{g} \right )$, where $\mathbf{W}_{s}$ and $\mathbf{M} \star \mathbf{W}_{s}$ are the parameters of the classification layer for the two networks, respectively, and $\mathbf{W}_{g}$ are the parameters of the layers for extracting features.

\begin{lemma} (Network Layer Bound~\cite{wan2013regularization}) Let $\mathcal{G}$ denote the class of $D$-dimensional real functions $\mathcal{G} = \ [ \mathcal{Q}_j\ ]_{j=1}^{D}$, and $\mathcal{H}$ denote a linear transform function $\mathcal{H}: {R}^{D}\rightarrow R$, which is parametrized by $\mathbf{W}$ with ${||\mathbf{W}||}_2 \leq {B}$; then $\hat{R}_{l}\left(\mathcal{H} \circ \mathcal{G} \right)$, the empirical Rademacher complexity of $\left(\mathcal{H} \circ \mathcal{G} \right)$,  satisfies $\hat{R}_{l}\left(\mathcal{H} \circ \mathcal{G} \right) \leq \sqrt{D} B  \hat{R}_{l}\left(\mathcal{Q}\right)$.\label{l1}
 \end{lemma} 
 
\begin{theorem} (Complexity of \emph{\textbf{OSLNet}}). Following the notation in Lemma~\ref{l1}, we further let $ \hat{R}_{l}\left(\mathcal{O}\right)$ denote the empirical Rademacher complexity of an OSLNet.
If the weights of the OSL meet $ |\mathbf{W}_s| \leq B_{s}$, we will have 
$ \hat{R}_{l}\left(\mathcal{O}\right) \leq \big ( \frac{D}{\sqrt{k}} B_{s} \big ) \hat{R}_{l}\left(\mathcal{Q}\right)$.

\end{theorem}
\begin{proof}
First, we denote the empirical Rademacher complexity of a standard network with the fully connected classification layer as $\hat{R}_{l}\left(\mathcal{O}_{0}\right) = \hat{R}_{l}\left( \mathcal{H}_{0} \circ \mathcal{G} \right)$. Since no connection has been removed in a standard network with the fully connected classification layer, $\mathcal{H}_{0}$ is a linear transform function $\mathcal{H}_{0}: {R}^{D}\rightarrow R$ and is parametrized by $\mathbf{W}_s$ with $|\mathbf{W}_s| \leq B_{s}$. Because $|\mathbf{W}_s| \leq B_{s}$ and the size of $\mathbf{W}_s$ is $D \times K$, we have $||\mathbf{W}_s||_2 \leq \sqrt{DK}{B}_s$. Thus, based on Lemma 1, we can obtain $ \hat{R}_{l}\left(\mathcal{O}_{0}\right) \leq \big ( D\sqrt{K} B_{s} \big ) \hat{R}_{l}\left(\mathcal{Q}\right)$. 

Similarly, for OSLNet, we have $\hat{R}_{l}\left(\mathcal{O}\right) = \hat{R}_{l}\left( \mathcal{H} \circ \mathcal{G} \right)$. Because in the OSLNet some connections are removed and the weight vectors of different classes are pairwise orthogonal, $\mathcal{H}$ becomes a linear transform function $\mathcal{H}$: $R^{\frac{D}{K}}\rightarrow R$ and  is parametrized by $\mathbf{M} \star \mathbf{W}_{s}$ ($||\mathbf{M} \star \mathbf{W}_{s}||_2 \leq \sqrt{D}{B}_s$). Therefore, based on Lemma 1, we have $ \hat{R}_{l}\left(\mathcal{O}\right) \leq \big ( \frac{D}{\sqrt{K}} B_{s} \big ) \hat{R}_{l}\left(\mathcal{Q}\right)$.
\end{proof}

The analysis above shows that the bound of empirical Rademacher complexity for the OSLNet is only $\frac{1}{K}$ of that for a standard network. In addition, the empirical error of the OSLNet is close to the standard model in terms of accuracy and cross-entropy loss on the training data, as shown in Fig.~\ref{loss}. Therefore, according to the relationship between a model generalization error bound and an empirical Rademacher complexity (Theorem 3.1 in~\cite{mohri2012foundations}), the OSLNet has a lower model generalization error bound than the standard network.

\section{Experimental Evaluation} \label{sec:Experiments}
To evaluate the performance of the proposed OSL, we compare OSLNet with other methods on four small-sample datasets and three large-sample datasets. These evaluation serve four purposes: 
\begin{enumerate}
\item To compare the proposed OSL with state-of-the-art methods on small-sample image classification (Sec.~\ref{Sec:sota}, Sec.~\ref{sec:Different Training Sizes}, and Sec.~\ref{sec:Paired Student's t-Test});
\item To investigate the effect of modifying the width or depth of the hidden layers on OSLNet (Sec.~\ref{sec:Effect on the Width} and Sec.~\ref{sec:Effect on the Depth}), and the effect of changing the feature extractor on OSLNet  (Sec.~\ref{sec:changing feature extractor});
\item To demonstrate the discriminability of the features learned from OSLNet (Sec.~\ref{sec:Features Visualization});
\item To demonstrate the performance of OSL on large-sample image classification (Sec.~\ref{sec:OSLNet on Large-sample Datasets}). 
\end{enumerate}

\subsection{Small-Sample Datasets}\label{sec:datasets}

\begin{table*}[h]
\caption{Comparison of the classification performance on the \emph{UIUC-Sports} (UIUC), \emph{15Scenes}, \emph{a subset of 80-AI} (80-AI), and \emph{Caltech101} datasets. The methods are: \emph{Fully connected network} (FC), \emph{Focal loss} (Focal), \emph{Center loss} (Center), \emph{Truncated $L_{q}$ loss} (T-$L_{q}$), \emph{Iterative Normalization} (IterNorm), \emph{Decorrelated Batch Normalization} (DBN), \emph{Dropout}, \emph{Large-Margin Softmax Loss} (Lsoftmax), \emph{SnapShot Ensembling} (SnapShot), and \emph{the proposed OSL} (OS), and \emph{SnapShot Ensembling of OS} (OS-SnapShot). Each method has been evaluated for $60$ rounds.} \label{A-10-A} 
\begin{center}

\begin{tabular}{c|c|ccccccccccc}\hline 
\bf {Datasets}  & \bf{Measure} &\bf {FC}  &\bf{Focal}  & \bf{Center} & \bf{T-$L_{q}$} & \bf{IterNorm} & \bf{DBN}  &\bf {Dropout} &\bf {Lsoftmax} & \bf {SnapShot}  & \bf {OS} &\bf{OS-SnapShot} 
\\ \hline
UIUC  &Mean                   &0.8837 & 0.8787 & 0.8347 & 0.8573 & 0.8506 &0.8378  &0.8889  &0.8946   &0.8950 &  0.9016  &\bf{0.9041}   \\
       &Std. &0.0151  & 0.0135 &0.0283  & 0.1693 & 0.0242 &  0.0752  &0.0113  &0.0076 & 0.0175 & 0.0055 &\bf{0.0030}  \\
\hline
15Scenes &Mean          &0.8331 &0.8285  & 0.7911 & 0.6551 & 0.8291 & 0.8005 &0.8321  &0.8434  & 0.8413  &0.8439  &\bf{0.8464}  \\
    & Std. &0.0080 & 0.0066 & 0.0152 &  0.3319 & 0.0067 & 0.0358  &0.0101  &0.0054 &0.0066  &0.0037 &\bf{0.0022} \\
    \hline
80-AI   &Mean      &0.5316 & 0.5291 & 0.4678 &-&0.5828  &0.4552  &0.5445  &0.3886   & 0.5825  &0.6157  &\bf{0.6192} \\
    & Std. &0.0139 & 0.0159 & 0.0356 & - & 0.0074 &0.0291   &0.0305  &0.0413 &0.0091  &0.0031  &\bf{0.0025}  \\
\hline
Caltech101    &Mean   &0.8927 &0.8881  &0.8644  & - & 0.8814 & 0.9254 &0.8865  &0.9168  & 0.9290  &0.9127  &\bf{0.9369 }\\
                     & Std. &0.0046 & 0.0047 & 0.0062 &-  & 0.0068 &  0.0028 &0.0077 &0.0071 &0.0020   &0.0044  &\bf{0.0011} \\
\hline
\end{tabular}
\end{center}
\end{table*}

\begin{figure*}[h]
\begin{center}
\includegraphics[width=7.0in]{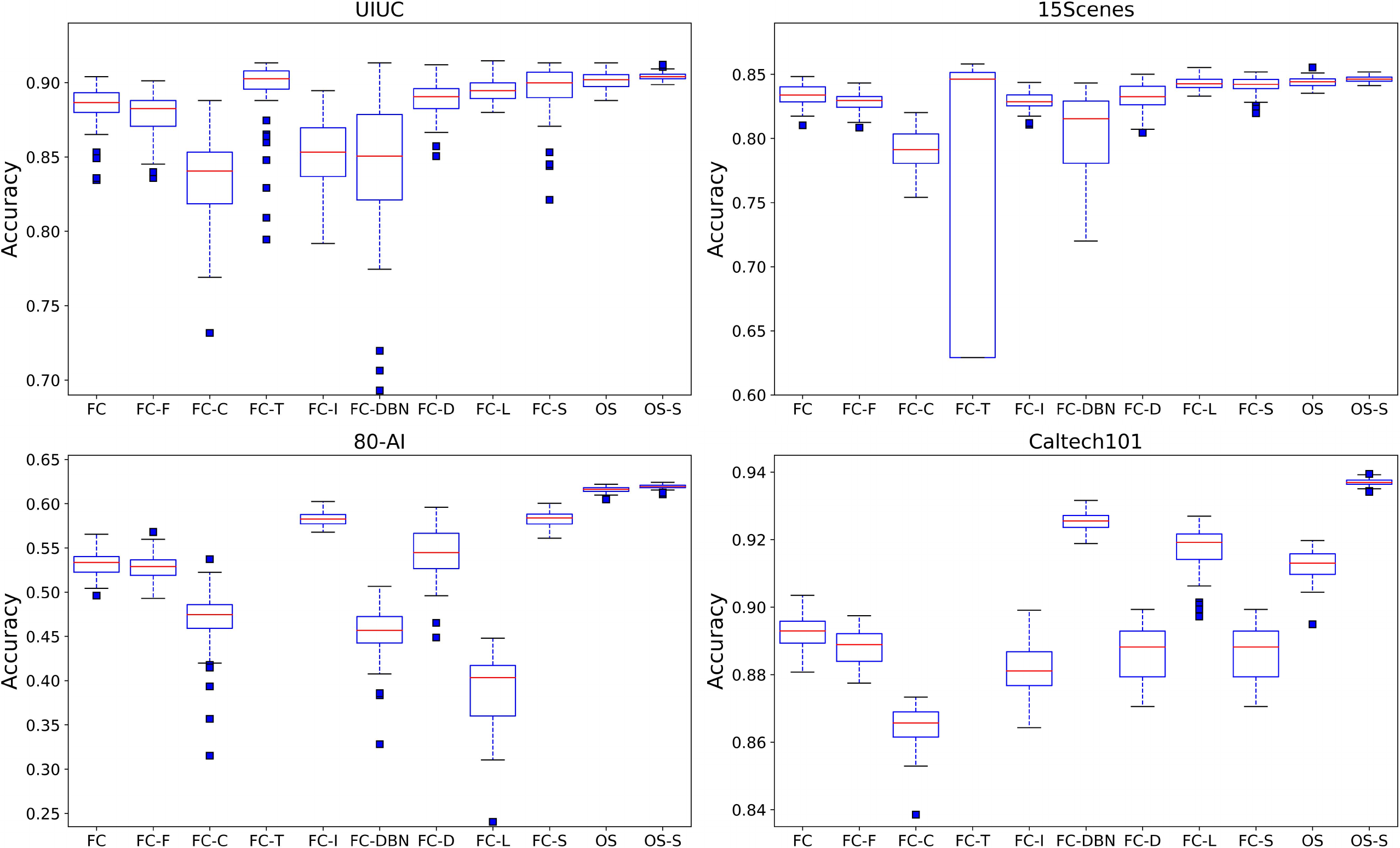}
\end{center}
\caption{Comparison of the accuracies obtained by FC, Focal (FC-F), Center (FC-C), T-$L_{q}$ (FC-T), IterNorm (FC-I), DBN (FC-DBN), Dropout (FC-D), Lsoftmax (FC-L), SnapShot (FC-S), OS, and OS-SnapShot (OS-S), via boxplots on the UIUC, 15Scenes, 80-AI, and Caltech101 datasets. The central mark is the median, and the edges of the box are the $25$th and $75$th percentiles. The outliers are individually marked. In the boxplots, each method has been evaluated for $60$ rounds.}\label{60times}
\end{figure*}

For experiments on small-sample image classification, we selected the following four small-sample datasets:
\begin{itemize}
\item UIUC-Sports dataset (UIUC)\footnote{http://vision.stanford.edu/lijiali/Resources.html}~\cite{li2007and}: This dataset contains $8$ classes of sports scene images. The total number of images is $1579$. The numbers of images for different classes are: bocce ($137$), polo ($182$), rowing ($250$), sailing ($190$), snowboarding ($190$), rock climbing ($194$), croquet ($236$), and badminton ($200$).

\item 15Scenes~\cite{lazebnik2006beyond}: This dataset contains $15$ classes of natural scene images: coast, forest, highway, inside city, mountain, open country, street and tall building. We randomly select $200$ images for each class, so the total number of images is $3000$. 

\item Subset of the Scenes dataset on AI Challenger (80-AI):  This dataset contains $80000$ images in $80$ classes of daily scenes, such as airport terminal and amusement park. The size of the classes is $600$-$1000$\ \footnote{https://challenger.ai/competition/scene/subject}. We randomly select $200$ images for each class, so the total number of images is $16000$. 

\item Caltech101~\cite{fei2004learning}: This dataset contains pictures of objects in 101 categories, and the size of each category is approximately 40-800 images. The total number of pictures in the dataset is $4000$. 
\end{itemize}

\begin{table*}[h]
\caption{Comparison of classification accuracies on the UIUC, 15Scenes, 80-AI and Caltech101 datasets when the training data are reduced. The notation DatasetName$-n$ denotes the configuration where the training data in the named dataset is reduced by $n$ data points for every class from the original training sets, whereas the test sets remain unchanged. Each method runs $60$ rounds on each dataset.} \label{reducing}
\begin{center}
\begin{tabular}{c|c|cccccc}\hline
  \bf {Datasets}   & \bf{Measure}  &\bf {FC}  &\bf {Dropout} &\bf {Lsoftmax} & \bf {SnapShot} & \bf {OS} &\bf{OS-SnapShot} 
\\ \hline
 UIUC-20       & Mean     &0.8715  &0.8776  &0.8800      &0.8822  &0.8891 &0.8894  \\
                      & Std.    &0.0134 &0.0146  &0.0082      &0.0129 &0.0074 &0.0080\\
\hline
UIUC-30     &Mean    &0.8577  &0.8631 &0.8682      &0.8682  &0.8793 &0.8692 \\
                    &Std.   &0.0135  &0.0145  &0.0087     &0.0142  &0.0060 &0.0118 \\
\hline
UIUC-40        &Mean     &0.8447  &0.8493  &0.8563      &0.8585  &0.8695 &0.8570 \\
                    &Std.   &0.0152  &0.0141  &0.0086      &0.0127 &0.0082 &0.0145 \\
\hline
UIUC-50         &Mean  &0.8296  &0.8349  &0.8372     &0.8353  &0.8461 &0.8329  \\
                      &Std.   &0.0115 &0.0139  &0.0092      &0.0174  &0.0065 &0.0223 \\
\hline 
\hline
    \bf {Datasets}         & \bf{Measure}   &\bf {FC}  &\bf {Dropout} &\bf {Lsoftmax} & \bf {SnapShot} & \bf {OS} &\bf{OS-SnapShot} 
\\ \hline
15Scenes-20          &Mean             &0.8283       & 0.8209    &0.8399    &0.8371  &0.8401   &0.8438   \\
               &Std.   &0.0067       &0.0174      &0.0070    &0.0077    &0.0039   &0.0022   \\
\hline
 15Scenes-40           &Mean                &0.8181  &0.8112      &0.8312    &0.8299   &0.8296     & 0.8347    \\
 &Std.      &0.0091  &0.0126      &0.0066    & 0.0067   &0.0039      &0.0026    \\

\hline
 15Scenes-60            &Mean                  &0.804  &0.7965  &0.8156   &0.8101    &0.8129   &0.8164    \\
         &Std.      &0.0090 &0.0120  &0.0077     &0.0066    &0.0042  &0.0028    \\
\hline
15Scenes-80           &Mean                    &0.7781   &0.7746   &0.7936  &0.7837 &0.7906 &0.7979 \\
                    &Std.         &0.0084   &0.0220 &0.0070   &0.0078 &0.0049 &0.0066  \\
\hline
\hline 
     \bf {Datasets}        & \bf{Measure}           &\bf {FC}  &\bf {Dropout} &\bf {Lsoftmax} & \bf {SnapShot} & \bf {OS} &\bf{OS-SnapShot} 
\\ \hline
80-AI-20          &Mean             &0.5169      & 0.5241     &0.4116   &0.5693   &0.5999  & 0.6023\\
               &Std.   &0.0141       &0.0401      &0.0217    &0.0076    &0.0027   &0.0040  \\
\hline
80-AI-40           &Mean                &0.4977 &0.5019      &0.4016    &0.5475    &0.5772      & 0.5680 \\
 &Std.      &0.0202  &0.0282      &0.0240    & 0.0080    &0.0031      &0.0118  \\

\hline
80-AI-60            &Mean                  &0.4766  &0.4605  &0.4022   &0.5099    &0.5386   &0.5089     \\
         &Std.       &0.0121  &0.0252  &0.0412     &0.0086    &0.0046  &0.0219   \\
\hline
80-AI-80           &Mean         &0.4112   &0.4018   &0.4043  &0.4247 &0.4624 &0.3696\\
                    &Std.         &0.0134   &0.0210  &0.0168   &0.0090 &0.0118 &0.0513\\
\hline
\hline 
   \bf {Datasets}    & \bf{Measure}           &\bf {FC}  &\bf {Dropout} &\bf {Lsoftmax} & \bf {SnapShot} & \bf {OS} &\bf{OS-SnapShot} 
\\ \hline
Caltech101-5    &Mean      & 0.8774  &  0.8760    & 0.9057   &  0.9163  & 0.9013  & 0.9261  \\
                     & Std.       &0.0048   &  0.0088  & 0.0059   &0.0028    &0.0050    &   0.0013  \\
\hline
Caltech101-10    &Mean      &0.8583   & 0.8579     &  0.8856  & 0.8976   & 0.8867  & 0.9037  \\
                     & Std.       &0.0047   & 0.0079   &0.0071    &  0.0035   & 0.0056   &  0.0015   \\
\hline
\end{tabular}
\end{center} 
\end{table*}

\begin{figure*}[h]
\begin{center}
\includegraphics[width=6.8in]{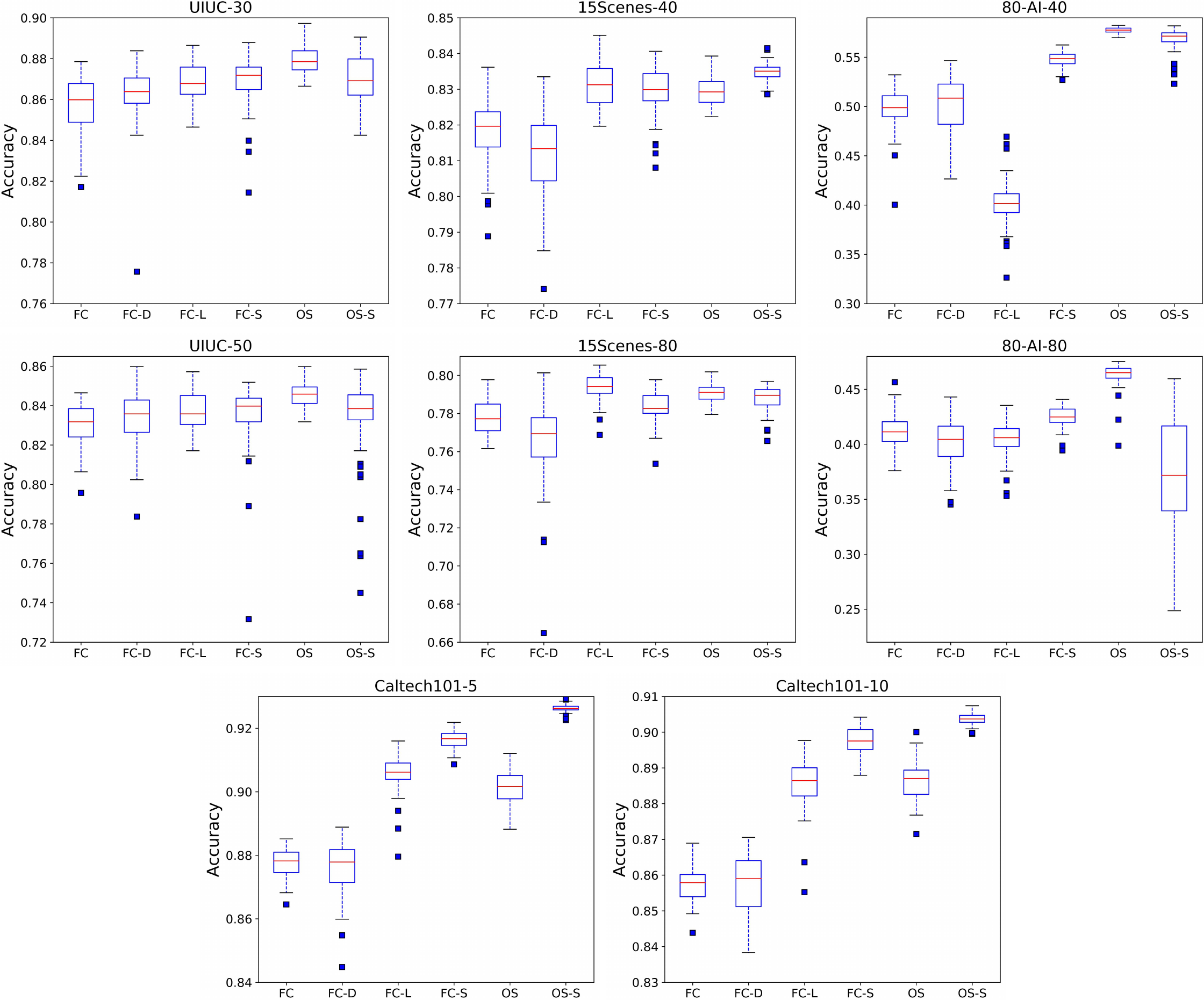}\label{1}
\end{center}

\caption{Comparison of accuracies obtained by FC, Dropout (FC-D), Lsoftmax (FC-L), SnapShot (FC-S), OS and OS-SnapShot (OS-S) via boxplots on the UIUC-30, UIUC-50, 15Scenes-40, 15Scenes-80, 80-AI-40, 80-AI-80, Caltech101-5 and Caltech101-10 datasets. In the boxplots, each method runs 60 rounds. }\label{reducing-60times}
\end{figure*}

For the 15Scenes and 80-AI datasets, in both the training and test datasets, each class contains $100$ samples. For the UIUC and Caltech101 datasets, unlike the 15Scenes and 80-AI datasets, different classes have different number of samples, and we randomly split the data of each class into the training and test sets evenly. 

We adopted a CNN feature extractor, VGG16~\cite{simonyan2014very}, which was pre-trained on ImageNet. First, we resized the images into identical sizes of $256\times 256$ and directly extracted the image features using the VGG16 network. Finally, we kept the features of the last convolutional layer and simply flattened them. The feature dimension of each image is $512 \times 8 \times 8 = 32768$.

\subsection{The Compared Methods and Their Implementation}\label{sec:Implementation}

To evaluate the classification performance of the proposed OSL, we compare the following methods: 1) fully connected network (FC); 2) Focal loss (Focal)~\cite{lin2017focal}; 3) Center loss (Center)~\cite{wen2016discriminative}; 4) Truncated $L_{q}$ loss (T-$L_{q}$)~\cite{zhang2018generalized}; 5) Iterative normalization (IterNorm)~\cite{huang2019iterative}; 6) Decorrelated batch normalization (DBN)~\cite{huang2018decorrelated}; 7) Dropout; 8) large-margin softmax loss (Lsoftmax)~\cite{liu2016large}; 9) SnapShot ensembling of FC (SnapShot)~\cite{huang2017SnapShot}; 10) the proposed OSL (OS); and 11) SnapShot ensembling of OS (OS-SnapShot). 

For FC, we used a fully connected network with two layers, where the activation functions of the first and second layers are rectified linear unit function (\emph{Relu}) and \emph{Softmax}, respectively. 
FC is optimized by minimizing the softmax cross-entropy loss based on the minibatch gradient descent. The optimization algorithm is the RMSprop with the initial learning rate $0.001$, the batch size is $32$, and the number of epochs is $100$.

For DBN and IterNorm, we placed them before each linear layer of FC. 
For DBN, we set group-number $=2$ for the UIUC and 15 scenes datasets and set group-number $=16$ for the 80-AI and Caltech101 datasets. For IterNorm, we set $T=10$ and group-number $=8$ for all the four datasets.

For Focal, Center, T-$L_{q}$, and Lsoftmax, we adopted the focal loss~\cite{liu2016large}, center loss, truncated $L_{q}$ loss, and large-margin softmax loss~\cite{liu2016large} to replace the softmax cross-entropy loss used in FC, respectively. For these four loss functions, we tried multiple sets of parameter values and selected the setting with best performance. Specifically, for Focal, we selected $\gamma =0.5$ for the 80-AI dataset and $\gamma =0.3$ for other three datasets. For Center, we selected a small 
loss weight, $\lambda=1e-10$. For T-$L_{q}$, we set $q=0.5$ and $k= 0.1$ for the UIUC dataset, and $q=0.1$ and $k= 0.1$ for the 15Scenes dataset. However, on the 80-AI and Caltech101 datasets, we did not find a set of $q$ and $k$ that makes T-$L_{q}$ fit the training data. For Lsoftmax, we set $m=2$. Other settings were kept the same as those for FC.

For Dropout, we added a Dropout layer after the hidden layer of FC. The probability that a neuron unit should be dropped is $0.5$. Other settings of Dropout are identical to FC. 

The SnapShot here is a SnapShot ensembling built on the FC network, where the learning rate is learned by using a cyclic cosine annealing method~\cite{loshchilov2016sgdr}, and the number of SnapShots is $2$. For OS, we replaced the fully connected classification layer with the proposed OSL in FC and other settings were identical to that of FC. For OS-SnapShot, except for the structure of base network, other settings are identical to that of SnapShot.

All the methods have been implemented with PyTorch~\cite{ketkar2017introduction}.

\subsection{Classification Accuracy}\label{Sec:sota}
We ran FC, Focal, Center, T-$L_{q}$, IterNorm, DBN, Dropout, Lsoftmax, SnapShot, OS and OS-SnapShot on the UIUC, 15Scenes, 80-AI and Caltech101 datasets for 60 rounds each. The mean and the standard deviation of the classification accuracy are listed in Table~\ref{A-10-A}, and the boxplot of classification accuracy is shown in Fig.~\ref{60times}. Larger mean and smaller standard deviation indicate better performance. 

Table~\ref{A-10-A} shows that, on the four datasets, FC is easy to overfit and has unstable performance. Moreover, Focal, Center, and T-$L_{q}$ underperform FC. IterNorm outperforms FC on 80-AI and DNB performs better than FC on Caltech101, and in other cases IterNorm and DBN underperform FC. Dropout performs slightly better than FC on UIUC, competitively with FC on 80-AI, and worse than FC on 15Scenes and Caltech101. Lsoftmax performs better than FC on UIUC, 15Scenes and Caltech101 but slightly worse than FC on 80-AI. SnapShot always has better performance than FC on the four datasets. OS performs better than FC, Dropout and Lsoftmax on UIUC, 15Scenes and 80-AI, showing larger mean and smaller variance. OS performs better than FC and Dropout, but worse than Lsoftmax and SnapShot, on Caltech101. We also evaluated the SnapShot ensembling version of OS, denoted as OS-SnapShot, and it obtains better performance on all the four datasets. 

\begin{table}[t]
\caption{The $p$-values of the proposed method (OS) and the referred methods, FC, Dropout, Lsoftmax and SnapShot, by the paired Student's t-test. Each method runs $60$ rounds on each dataset. }
 \label{test}
\begin{center}
\begin{tabular}{l|llll}\hline
 \bf {Datasets}  &\bf {FC}  &\bf {Dropout} &\bf {Lsoftmax} & \bf {SnapShot} 
 \\ \hline
UIUC           &$<$ 0.005   & $<$  0.005     &$<$ 0.005 &0.0116    \\
 UIUC-20        &$<$  0.005      &$<$  0.005           &$<$  0.005     & $<$  0.005    \\
UIUC-30        &$<$  0.005       &  $<$  0.005      &$<$  0.005     & $<$  0.005  \\
UIUC-40          & $<$  0.005       &$<$  0.005      &  $<$ 0.005   & $<$  0.005      \\
UIUC-50      & $<$  0.005       & $<$  0.005         &$<$  0.005   & $<$  0.005   \\
\hline
15S.      &$<$ 0.005     & $<$ 0.005     & 0.6014   & 0.0087    \\
15S.-10      &$<$ 0.005      & $<$ 0.005   & 0.8284  &  0.0057   \\
15S.-30      & $<$ 0.005      & $<$ 0.005     &0.1437   & 0.7747  \\ 
15S.-50         & $<$ 0.005   &  $<$ 0.005    &0.0225  & 0.0090  \\
15S.-70     &$<$ 0.005       &$<$ 0.005      & 0.0093  & $<$ 0.005    \\
\hline 
80-AI        & $<$  0.005       & $<$  0.005         &$<$  0.005   & $<$  0.005    \\
80-AI-10      & $<$  0.005       & $<$  0.005         &$<$  0.005   & $<$  0.005    \\
80-AI-30  & $<$  0.005       & $<$  0.005         &$<$  0.005   & $<$  0.005      \\ 
80-AI-50      & $<$  0.005       & $<$  0.005         &$<$  0.005   & $<$  0.005   \\                  
80-AI-70    & $<$  0.005       & $<$  0.005         &$<$  0.005   & $<$  0.005     \\
\hline 
Cat.101       & $<$  0.005       & $<$  0.005         &$<$  0.005   & $<$  0.005     \\
Cat.101-5       & $<$  0.005       & $<$  0.005         &$<$  0.005   & $<$  0.005     \\
Cat.101-10     & $<$  0.005       & $<$  0.005         &$<$  0.005   & $<$  0.005    \\
\hline
\end{tabular}
\end{center} 
\end{table}

\begin{table}[t]
\caption{Classification accuracies obtained by FC and OS on the UIUC and 15Scenes (15Sce.) datasets when we change the size of the hidden layer. FC and OS have 2 layers and the size of their hidden layer is shown. Each setting of FC and OS has been evaluated for $60$ rounds.} \label{width}
\begin{center}
\begin{tabular}{l|llllll}\hline
\bf {UIUC} & \bf{16} & \bf{32}  &\bf {64}  &\bf {128} &\bf {256} & \bf {512} 
\\ \hline
FC-Mean     &  0.8525    & 0.8837 &0.8949  &0.8800      &0.8973  &0.8945  \\
FC-Std.    & 0.0218 &0.0151 &0.0146  &0.0074      &0.0073 &0.0096 \\
\hline
OS-Mean                    &0.9002 &0.9016 &0.9022 &0.8987      &0.8877 &0.8836 \\
OS-Std.  &0.0074 &0.0054  &0.0056  &0.0055     &0.0070 &0.0065  \\
\hline 
\hline
\bf{15Sce. }  & \bf{16}  & \bf{32}  &\bf {64}  &\bf {128} &\bf {256} & \bf {512} 
\\ \hline
FC-Mean    &N/A      &0.8070   & 0.8331   &0.8409   &0.8415  &0.8431   \\
FC-Std.   &N/A    &0.0131   &0.0080      &0.0062    &0.0059 &0.0061 \\
\hline
OS-Mean  & N/A        &0.8419   &0.8439      &0.8429    &0.8401  &0.8373  \\
OS-Std.    &N/A   &0.0043  &0.0037   &0.0038    & 0.0036   &0.0040 \\
\hline
\end{tabular}
\end{center} 
\end{table}

Fig.~\ref{60times} demonstrates that on the UIUC-Sports dataset, the boxplot of OS is more compact than those of FC, Focal, Center, T-$L_{q}$, IterNorm, DBN, Dropout, Lsoftmax and SnapShot. Moreover, it has no bad-performing outlier. On the 15Scenes dataset, the boxplot of OS is more compact than those of FC, Dropout and SnapShot, but is close to that of Lsoftmax. On the 80-AI dataset, the boxplot of OS is also more compact than those of other methods, and both the central mark and the edges of the box are higher than those of other methods used for comparison. On the Caltech101 dataset, the boxplot of OS is also more compact and higher than those of FC and Dropout, but is lower than those of Lsoftmax and SnapShot. Finally, the boxplots of OS-SnapShot is always more compact than those of all other compared methods, with both the central mark and the edges of the boxes higher.

\subsection{Classification Accuracy for Different Training Set Sizes}\label{sec:Different Training Sizes}

Table~\ref{reducing} and Fig.~\ref{reducing-60times} show that, on the UIUC datasets with decreased training set sizes, Dropout shows small improvement over FC. Lsoftmax achieves larger mean and lower variance than FC and Dropout. SnapShot, an ensemble method, obtains larger mean, but the variance is similar to the others. It is encouraging to see that OS has larger mean and lower standard deviation than all the four compared methods on all the different sizes of training sets. OS-SnapShot outperforms OS in terms of the mean accuracy. 

On the 15Scenes dataset, Dropout performs worse than FC, in terms of variance and mean. Lsoftamx consistently performs better than FC. SnapShot shows slight improvement but is worse than Lsoftmax. OS still performs well, as expected. OS-SnapShot improves the mean compared with OS. 

On the 80-AI dataset, Lsoftmax and Dropout perform worse than FC. Lsoftmax notably is difficult to converge with the four training sizes. SnapShot performs better than FC, Dropout and Lsoftmax. OS has smaller variances and larger means and performs much better than all other methods. 

On the Caltech101 dataset, SnapShot performs better than FC, Dropout and Lsoftmax. OS has smaller variances and larger means than FC and Dropout and has similar performance to Lsoftmax. OS-SnapShot performs the best. 

In summary, the proposed OSL performs better than other compared methods on these datasets.

\begin{table*}[h]
\caption{Comparison of the classification accuracies on the UIUC and 15Scenes datasets when we vary the depth of the neural networks. The sizes of the hidden layers are listed. Each structure of FC and OS is tested for $60$ rounds; the mean and standard deviation (Std.) are listed in the cells of the table.} \label{depth}
\begin{center}
\begin{tabular}{l|l|lllll}\hline
\bf {UIUC}    & \bf{Measure}   & \bf{32-8}  &\bf {64-32-8}  &\bf {128-64-32-8}  &\bf{256-128-64-32-8} &\bf {512-256-128-64-32-8} 
\\ \hline
FC       & Mean                        &0.8837  &0.8834  &0.8796     &0.8687 & 0.8545\\
                      & Std.   &0.0151 &0.0101  &0.0096     &0.0101  & 0.0129 \\
\hline
OS     &Mean                       &0.90167  &0.8933 &0.8872     &0.8575 & 0.8482\\
                    &Std.         &0.0055  &0.0089  &0.0095     &0.0136   &0.0169\\
\hline 
\hline
\bf{15Scenes }        & \bf{Measure}       &   &\bf {64-15}  &\bf {128-64-15} &\bf{256-128-64-15} &\bf {512-256-128-64-15} 

\\ \hline
FC        &Mean             &       & 0.8331    &0.8378    & 0.8391   &0.8288  \\
&Std.           &       &0.0080     &  0.0061  &0.0056   &0.0077   \\
\hline
OS         &Mean            &  &0.8439      &  0.8406  & 0.8404  &0.8208 \\
&Std.                  &       &0.0037      & 0.0062  &0.0057   &0.0083\\
\hline
\end{tabular}
\end{center} 
\end{table*}

\begin{table*}[h]
\caption{Comparison of the classification performance on the UIUC and 15Scenes datasets, with $60$ rounds of valuations for each method.} \label{feature-extractor} 
\begin{center}
\begin{tabular}{c|c|ccccccccc}\hline 
\bf {UIUC }  & \bf{Measure} &\bf {FC} &\bf{Focal}  & \bf{Center} & \bf{IterNorm} & \bf{DBN}  &\bf {Dropout}  & \bf {SnapShot}  & \bf {OS} &\bf{OS-SnapShot} 
\\ \hline
   VGG16 &Mean    & 0.8787 & 0.8347 & 0.8506 &0.8378  &0.8837   &0.8889    &0.8950 &  0.9016  &\bf{0.9041}   \\

       &Std. &0.0151   & 0.0135 &0.0283 & 0.0242 &  0.0752   &0.0113  & 0.0175 & 0.0054 &\bf{0.0030}  \\
\hline
 DenseNet121 &Mean   &0.5532 &  0.5389& 0.3248  &0.3163 &0.5234 &0.4697  &0.6225&  0.6438  &\bf{0.6954}   \\

       &Std.             &0.0631 & 0.0797 & 0.1105  &0.0560 &0.1197 & 0.0891    & 0.0391 & 0.0633 &\bf{0.0158}  \\
\hline
\hline
\bf {15Scenes}  & \bf{Measure} &\bf {FC} &\bf{Focal}  & \bf{Center} & \bf{IterNorm} & \bf{DBN}  &\bf {Dropout}  & \bf {SnapShot}  & \bf {OS} &\bf{OS-SnapShot} 
\\ \hline
 VGG16 &Mean    &0.8331 &0.8285  & 0.7911  & 0.8291 & 0.8005  &0.8321   & 0.8413  &0.8439  &\bf{0.8464}  \\
    & Std. &0.0080  & 0.0007 & 0.0152  & 0.0007 & 0.0357  &0.0110   &0.0066  &0.0037 &\bf{0.0022} \\
\hline
 DenseNet121 &Mean     &0.5783  &0.5653 & 0.3497 & 0.4836 & 0.5694 &0.5443 & 0.6366  &0.6200  &\bf{0.6799}  \\
   & Std.                  &0.0465  &0.0490 & 0.1029 & 0.1158 & 0.0546 &0.0379  &0.0154  &0.0359 &\bf{0.0045} \\
\hline
\end{tabular}
\end{center}
\end{table*}

\subsection{Paired Student's t-test}\label{sec:Paired Student's t-Test}
The experimental results presented in the previous sections show that the proposed OSL obtains better performance. To confirm that the improvement is statistically significant, we performed a paired Student's t-test~\cite{zimmerman1997teacher} for OS vs. FC, OS vs. Dropout, OS vs. Lsoftmax, and OS vs. SnapShot, and the $p$-values are listed in Table~\ref{test}. Following~\cite{singh2017big}, we set the significance level as $0.05$. 

For the tests of OS vs. FC, all the $p$-values are much smaller than the significance level. Thus, the null hypothesis that FC and OS have the identical mean is always rejected. Likewise, the null hypothesis that Dropout has the identical mean to OS is also rejected. 

In terms of Lsoftmax, on the UIUC, 80-AI and Caltech101 datasets, the null hypothesis that Dropout and OS have the identical mean is always rejected. However, on the 15Scenes dataset, the $p$-values are larger than the significance level in most cases. Thus, on 15Scenes, the null hypothesis that Lsoftmax and OS have the identical mean is generally not rejected.

For SnapShot, on the 80-AI and Caltech101 datasets, the null hypothesis that SnapShot and OS have the identical mean is always rejected. However, on the UIUC and 15Scenes datasets, the null hypothesis that SnapShot has the identical mean with OS is not always rejected. 

\subsection{Effect of Changing the Width of the Hidden Layers on OSLNet}\label{sec:Effect on the Width}

In the previous experiments, our method and all the baseline methods used the network structure with $32$ neurons in the last hidden layer. To explore the effect of the width of the hidden layer on the performance of OSLNet, we changed the number of hidden neurons in both FC and OS from $16$ to $512$ on the UIUC dataset, with other settings unchanged. At the same time, we also changed the number of the hidden neurons from $32$ to $512$ on the 15Scenes dataset, with other settings unchanged. FC and OS were evaluated on the two datasets for $60$ rounds, respectively.

From Table~\ref{width}, we can observe that, with the increase in the width of network, the mean of FC has generally slight improvement on UIUC, and the standard deviation decreases. That is, FC generally performs better when the number of the hidden neurons is increased. In contrast, OS performs better when the number of the hidden neurons is small. Particularly, when the number of the hidden neurons is $16$, OS is much superior to FC. Moreover, on 15Scenes, a similar pattern can be observed. 
In addition, OS has better performances when approximately $2$ to $4$ hidden neurons are assigned to each class on these two datasets with 8 and 15 classes, respectively.  

\subsection{Effect of Changing the Depth of the Hidden Layers on OSLNet} \label{sec:Effect on the Depth}
We also evaluate the effect of the depth of the network on OSLNet. In particular, we varied the depth of the network from $2$ to $5$ layers and from $2$ to $4$ layers in both FC and OS on the UIUC and 15Scenes datasets, respectively. Each structure of FC and OS is evaluated on UIUC and 15Scenes for $60$ rounds. The size of each layer, the depth of each layer, the corresponding mean and standard deviation of accuracy are listed in Table~\ref{depth}. 

From Table~\ref{depth}, we can observe that, on UIUC, the mean values of FC and OS become smaller, when the depth is increased, that is, the performance of both FC and OS decreases, and OS decreases faster than FC. A similar pattern is also found on 15Scenes. Therefore, shallow structures of both FC and OS perform better than deep ones of FC and OS on the two datasets.

\subsection{Effect of Changing Feature Extractor on OSLNet} \label{sec:changing feature extractor}
In all the experiments presented above, we used a pre-trained VGG16 as the feature extractor. In this section, we changed the feature extractor to a pre-trained DenseNet121 and ran all the compared methods on the UIUC and 15Scenes datasets for $60$ rounds each. The classification results are shown in Table~\ref{feature-extractor}. Here, we do not list the performance of T-$L_{q}$ and Lsoftmax, as it cannot fit the training data if the feature extractor is DenseNet121.

From Table~\ref{feature-extractor}, we can observe that: 1) On the UIUC and 15scenes datasets, each method performs better with the VGG16 feature extractor than with the DenseNer121 feature extractor. 2) When the feature extractor is DenseNet121, on the UIUC dataset, the proposed OS outperforms all the compared methods, and on the 15scenes dataset, OS underperforms SnapShot
and outperform other compared methods. Then we can conclude that OS-SnapShot performs best on both datasets.

In summary, the performance comparision on DenseNet121 is similar to those on VGG16, which further shows the applicability of the OSL.

\begin{figure*}[htbp]
\begin{center}
\includegraphics[width=1.7in]{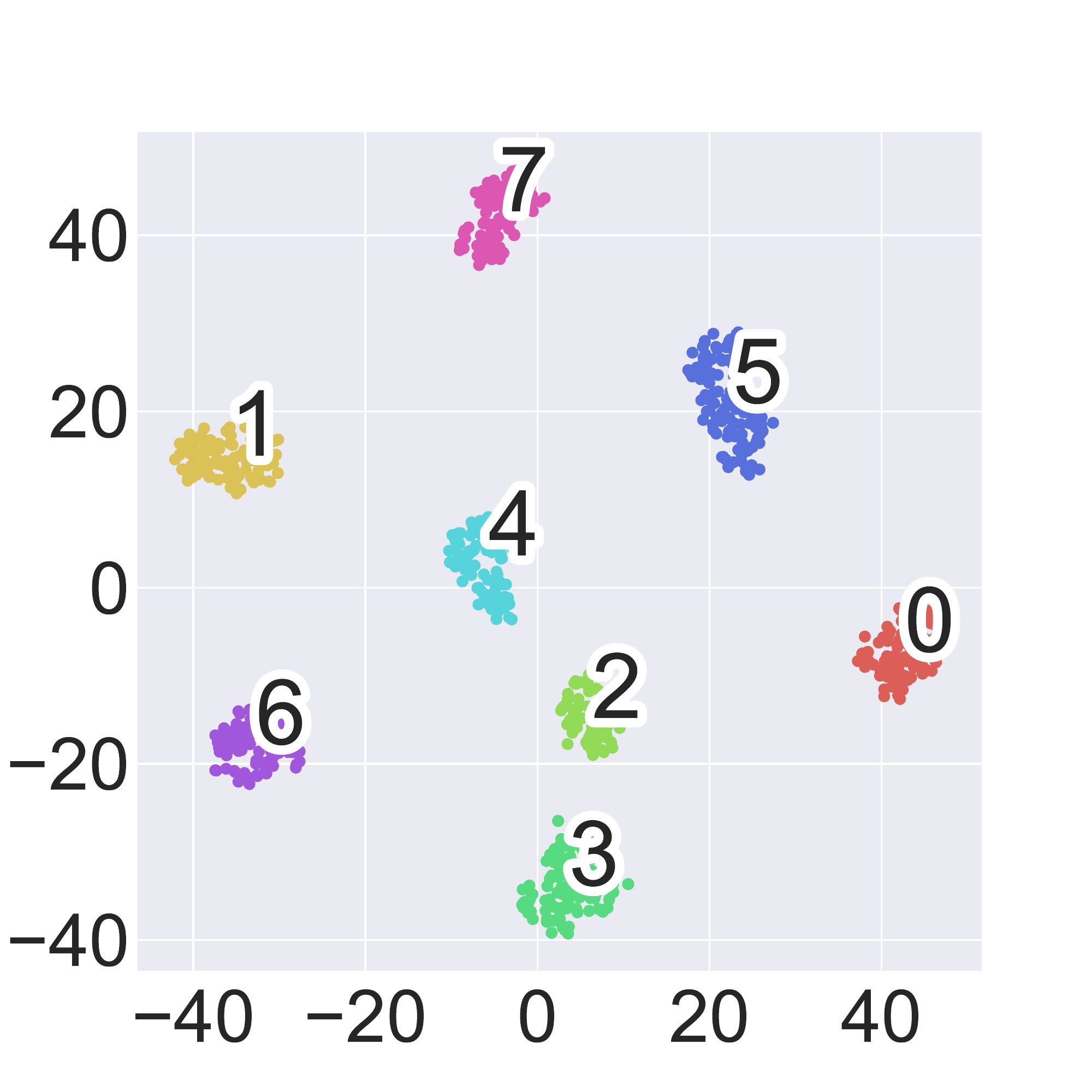}\label{train-0}
\includegraphics[width=1.7in]{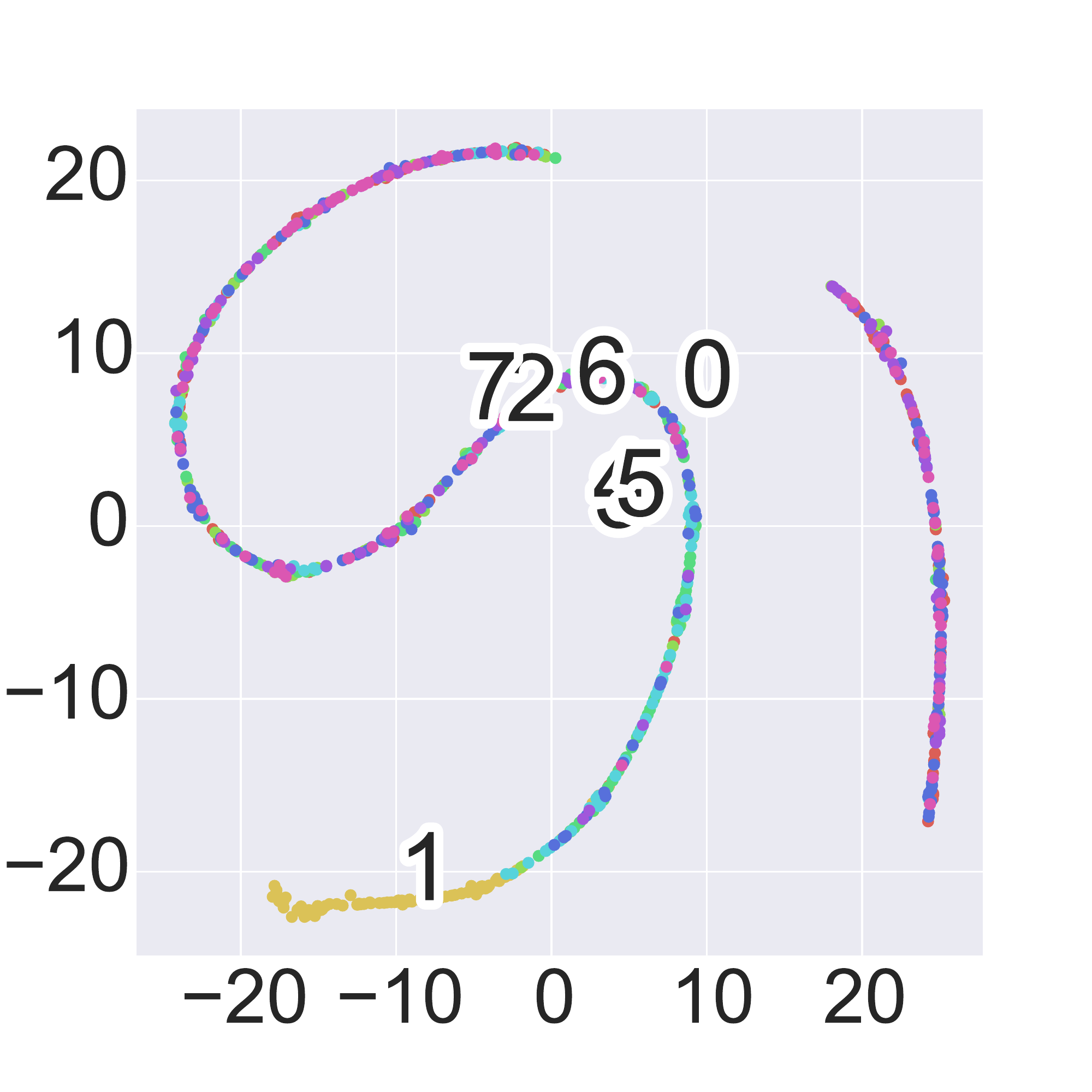}\label{9}
\includegraphics[width=1.7in]{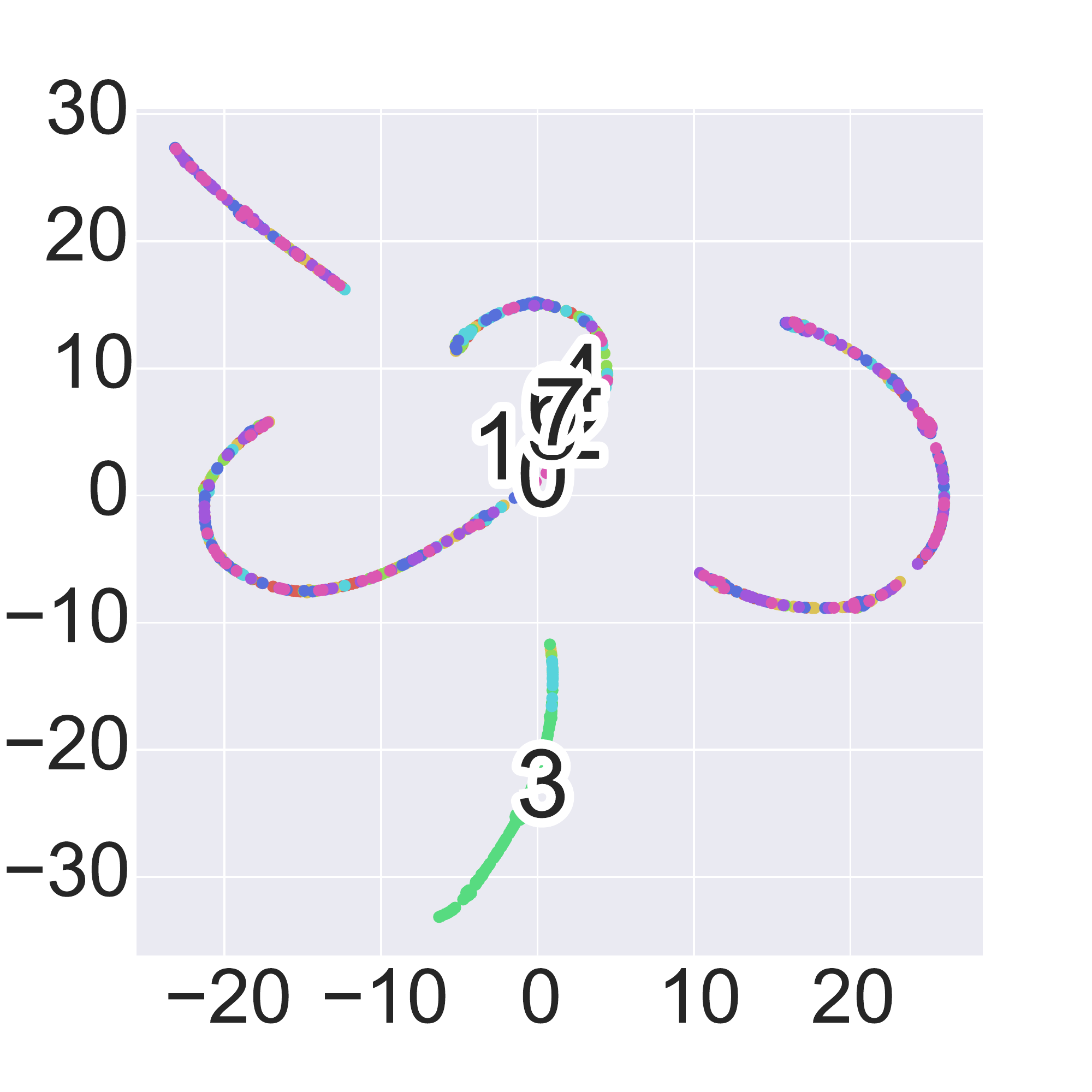}\label{10}
\includegraphics[width=1.7in]{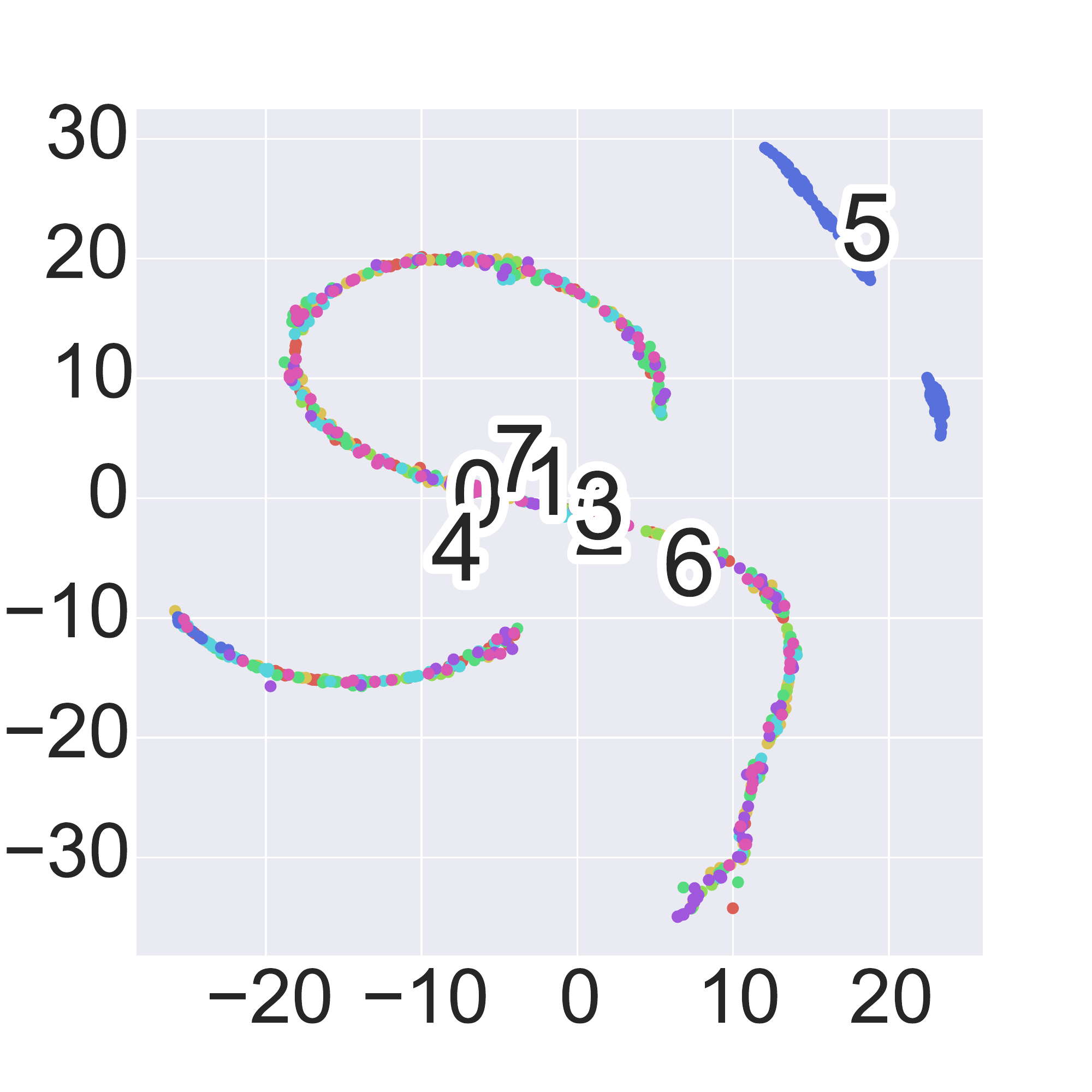}\label{11}
\includegraphics[width=1.7in]{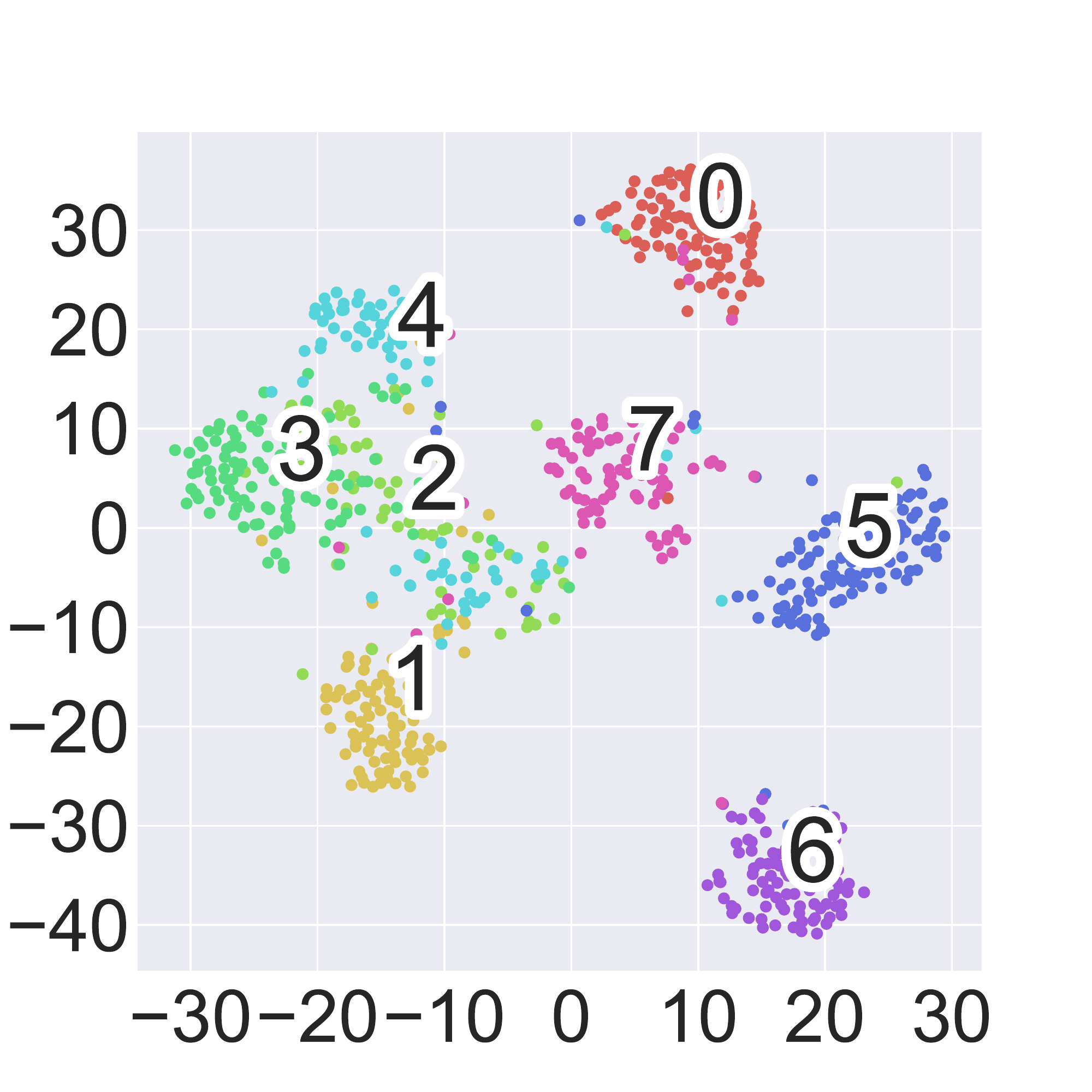}\label{test-0}
\includegraphics[width=1.7in]{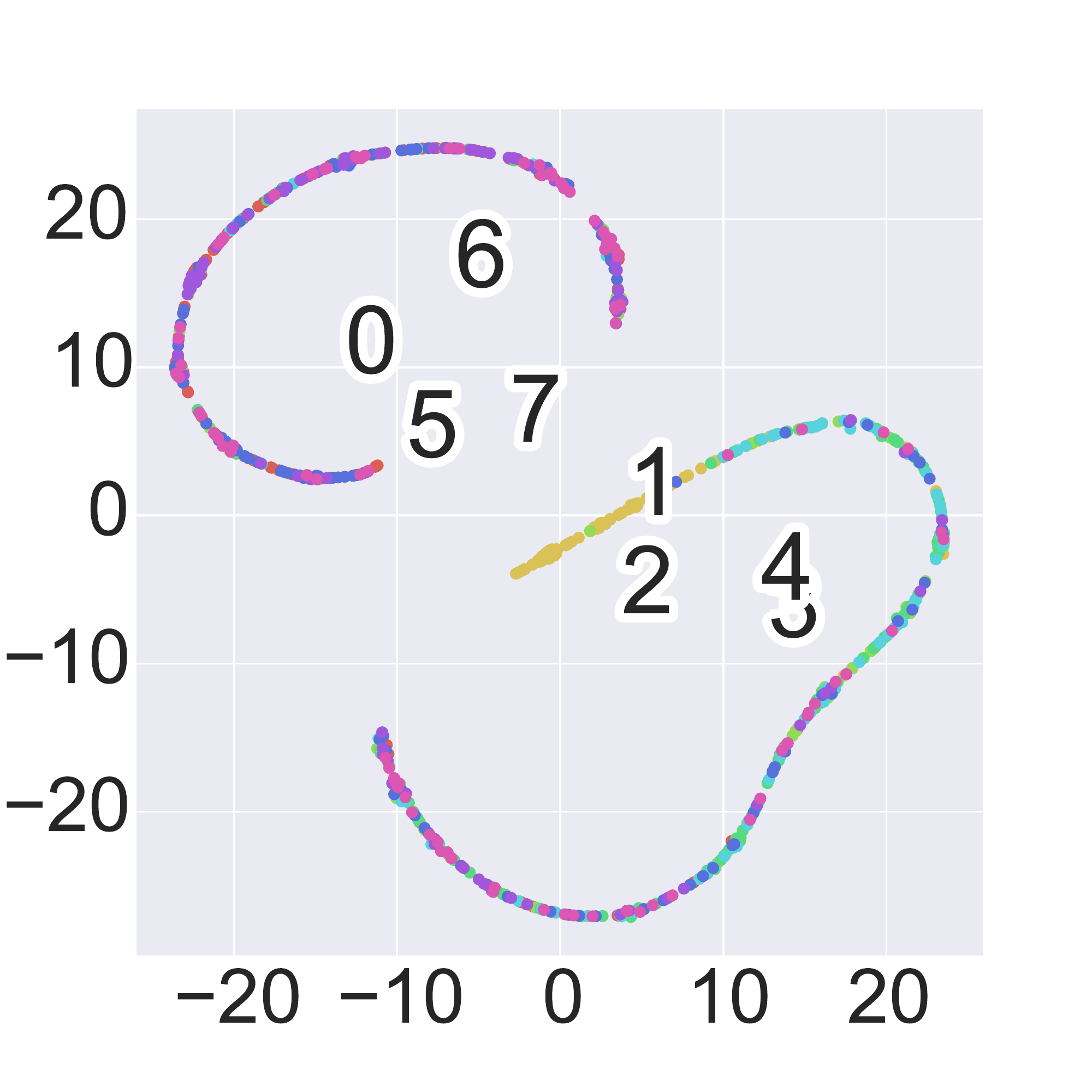}\label{13}
\includegraphics[width=1.7in]{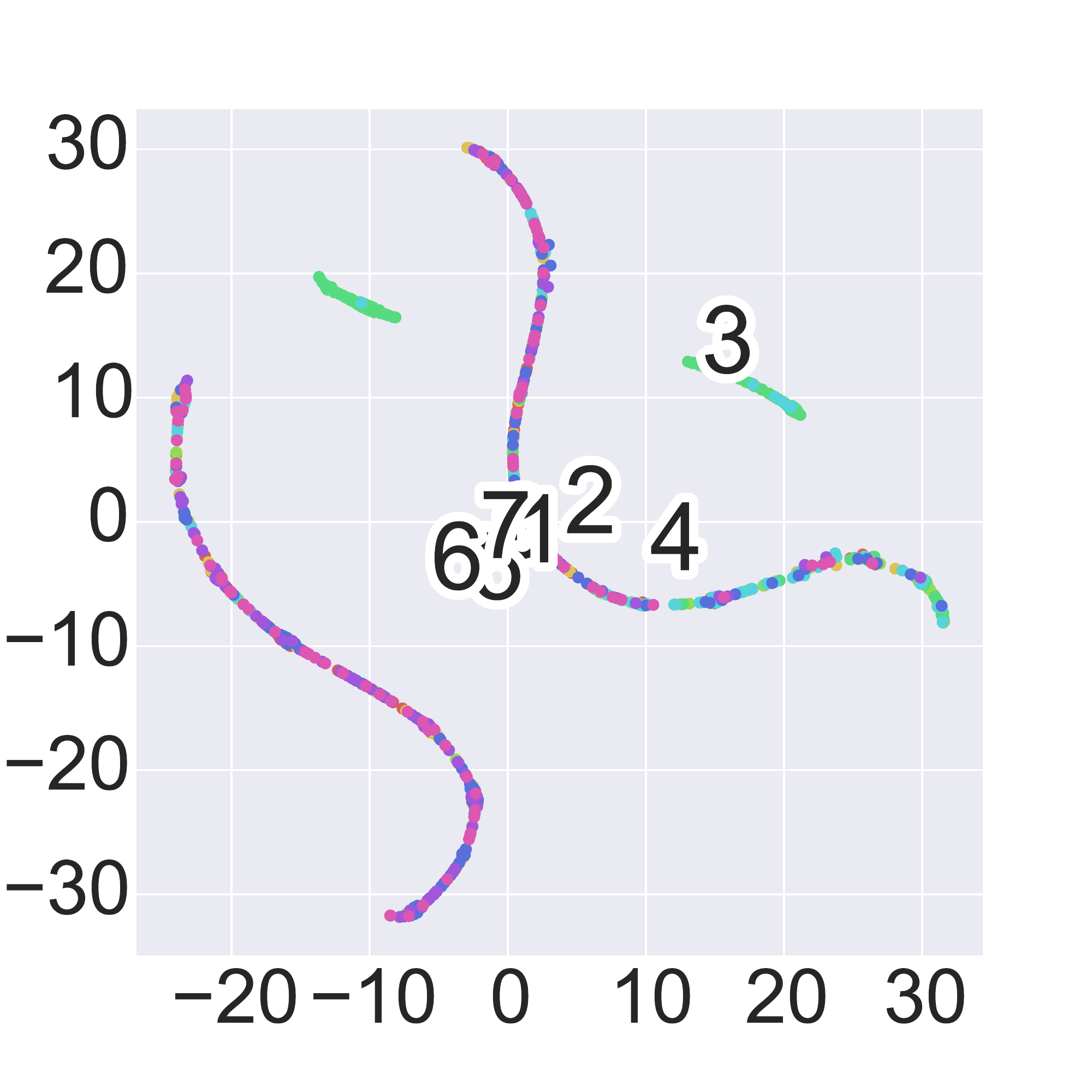}\label{14}
\includegraphics[width=1.7in]{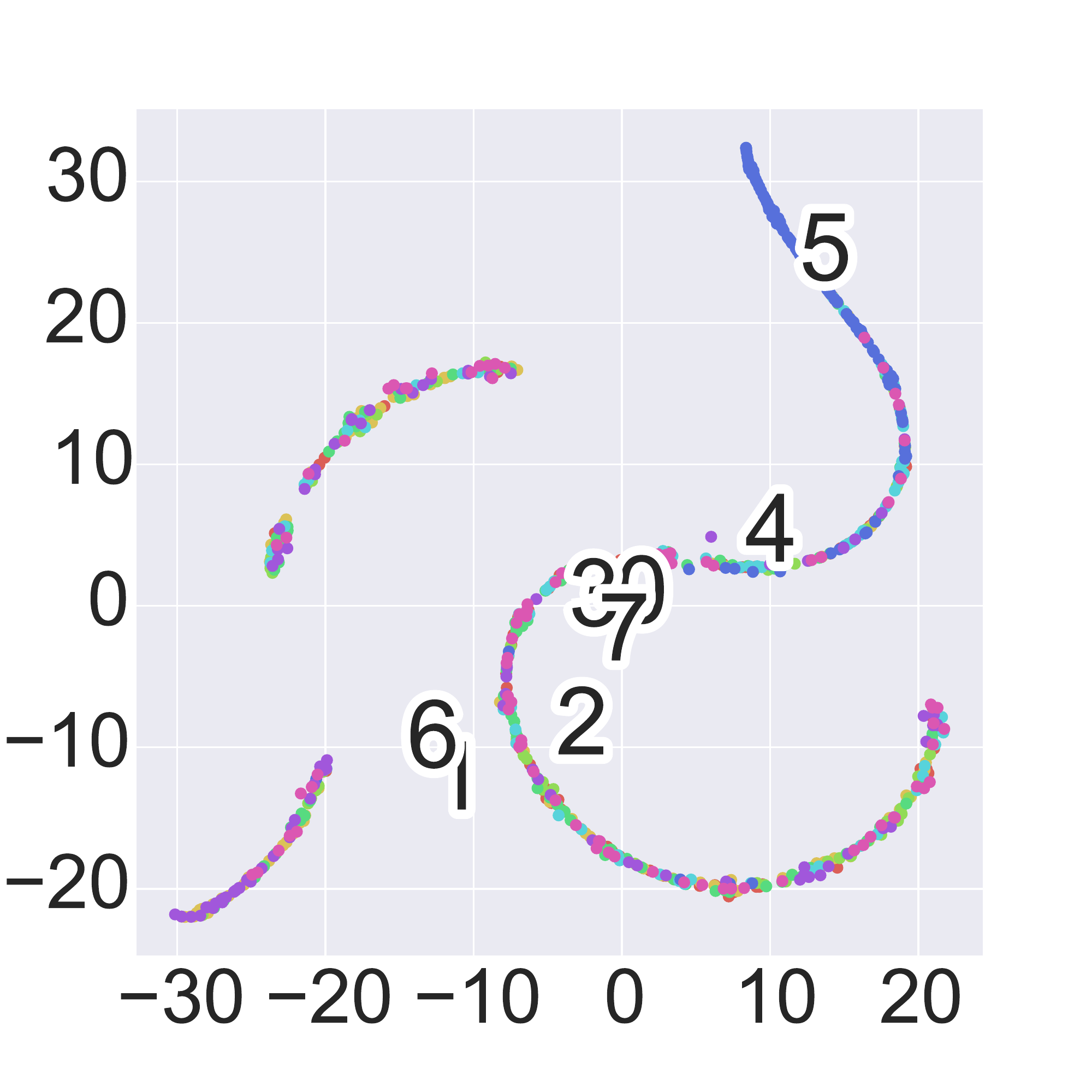}\label{15}
\end{center}
\caption{Visualization of the input feature of the OSL via t-SNE. The four columns from left to right are: 1) the $32$-dimensional input feature (without truncation) of OSL, 2) the 4th-8th dimensional features, 3) the 13th-16th dimensional features, and 4) the 21st-24th dimensional features, of which the last three columns correspond to the classes 1, 3 and 5 in the OSL, respectively. Upper row: the training data. Lower row: the test data.
}\label{feature}
\end{figure*} 

\subsection{Feature Visualization}\label{sec:Features Visualization}
To gain insights on the proposed OSL, we visualize the input features of OSL. Take the experiments on the UIUC dataset as an example: the input dimension of OSL is $32$, and the output dimension of OSL is $8$, as determined by the number of classes. Therefore, according to the design of OSL, each output neuron (i.e., each class) is assigned $4$ input neurons. For example, the 1st-4th dimensional input features are for the class $0$, and the 5th-8th dimensional input features are for the class $1$ (The class labels are from $0$ to $7$ on the UIUC dataset).
We use t-SNE~\cite{maaten2008visualizing} to visualize the input features of OSL in Fig.~\ref{feature}. In Fig.~\ref{feature}, the first column is for the 32-dimensional input features without truncation, and the following three columns are for the 4th-8th dimensional features, the 13th-16th dimensional features, the 21th-24th dimensional features, which correspond to the classes 1, 3 and 5 in the OSL, respectively.

From Fig.~\ref{feature}, it can be observed that, in the first column, the samples from different classes are separate. In the following three columns, the samples of the corresponding class are put to one end of a feature curve or a place far from the features of other classes. A similar pattern appears for other classes in the UIUC dataset and in other datasets. This demonstrates that, in the proposed OSL, the input neurons assigned to each class are indeed responsible for learning discriminative features for differentiating that class from other classes.

\begin{figure*}[t]
\begin{center}
\includegraphics[width=2.5in]{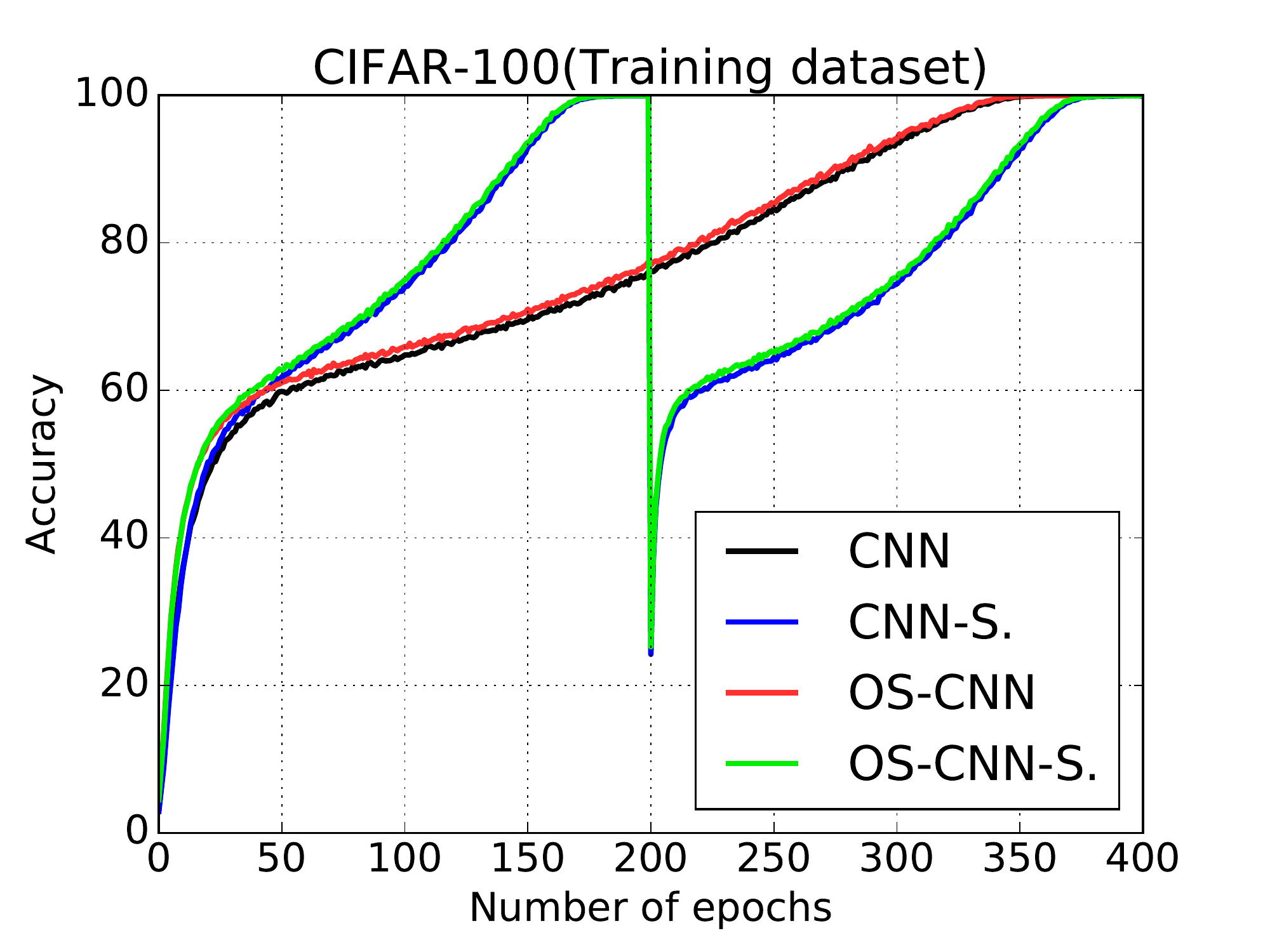}\label{17}
\includegraphics[width=2.5in]{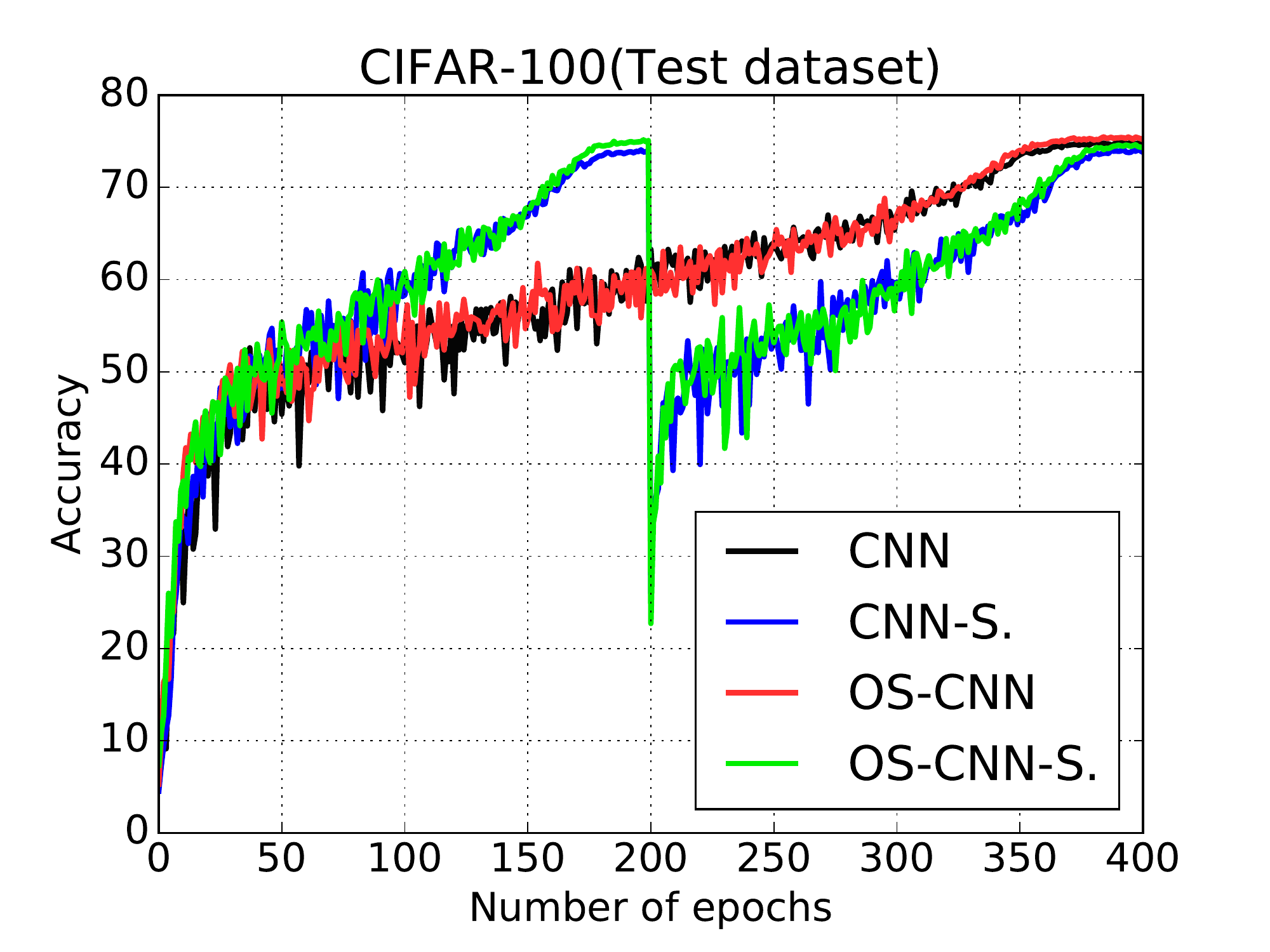}\label{18}
\includegraphics[width=2.5in]{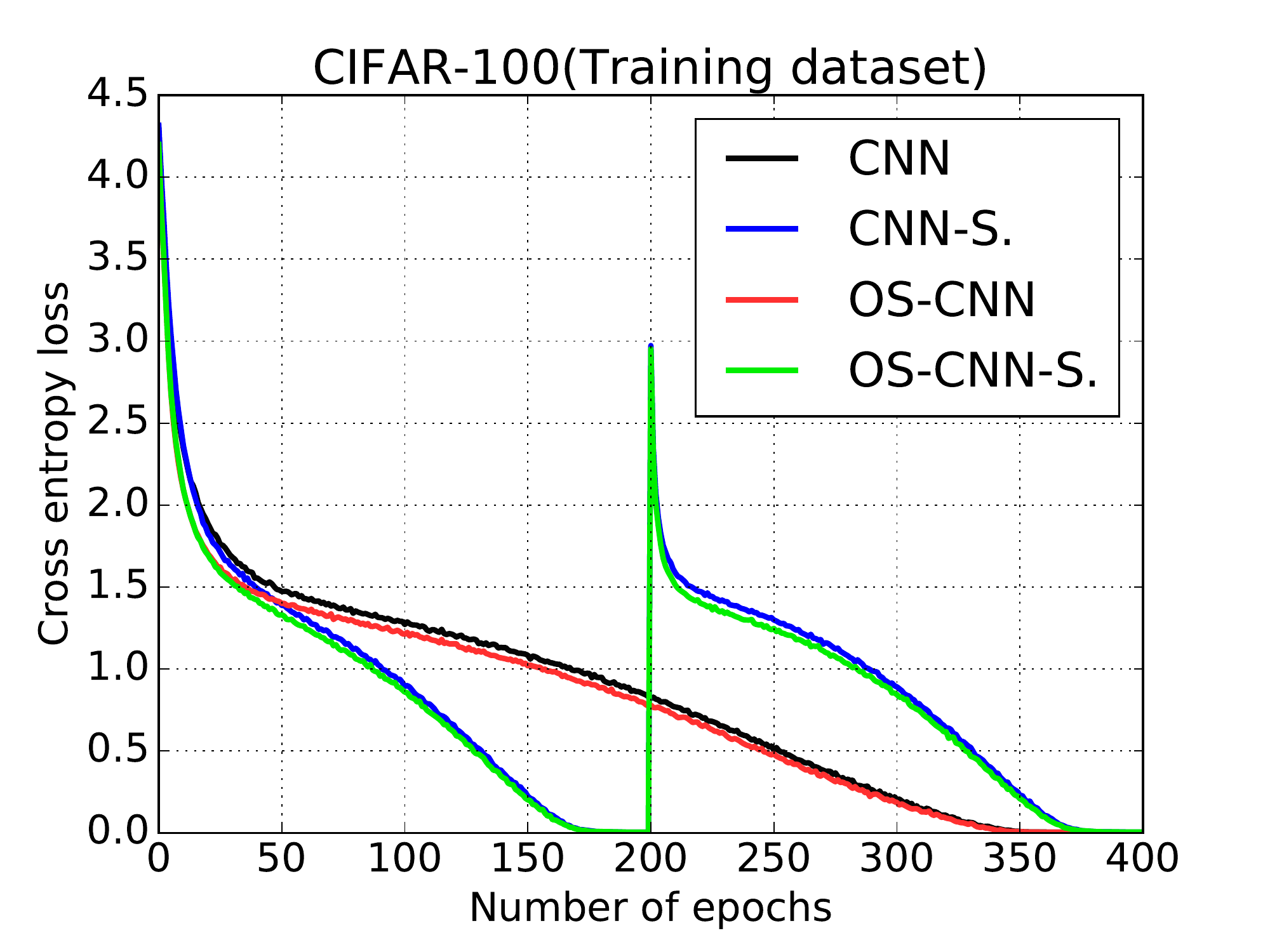}\label{19}
\includegraphics[width=2.5in]{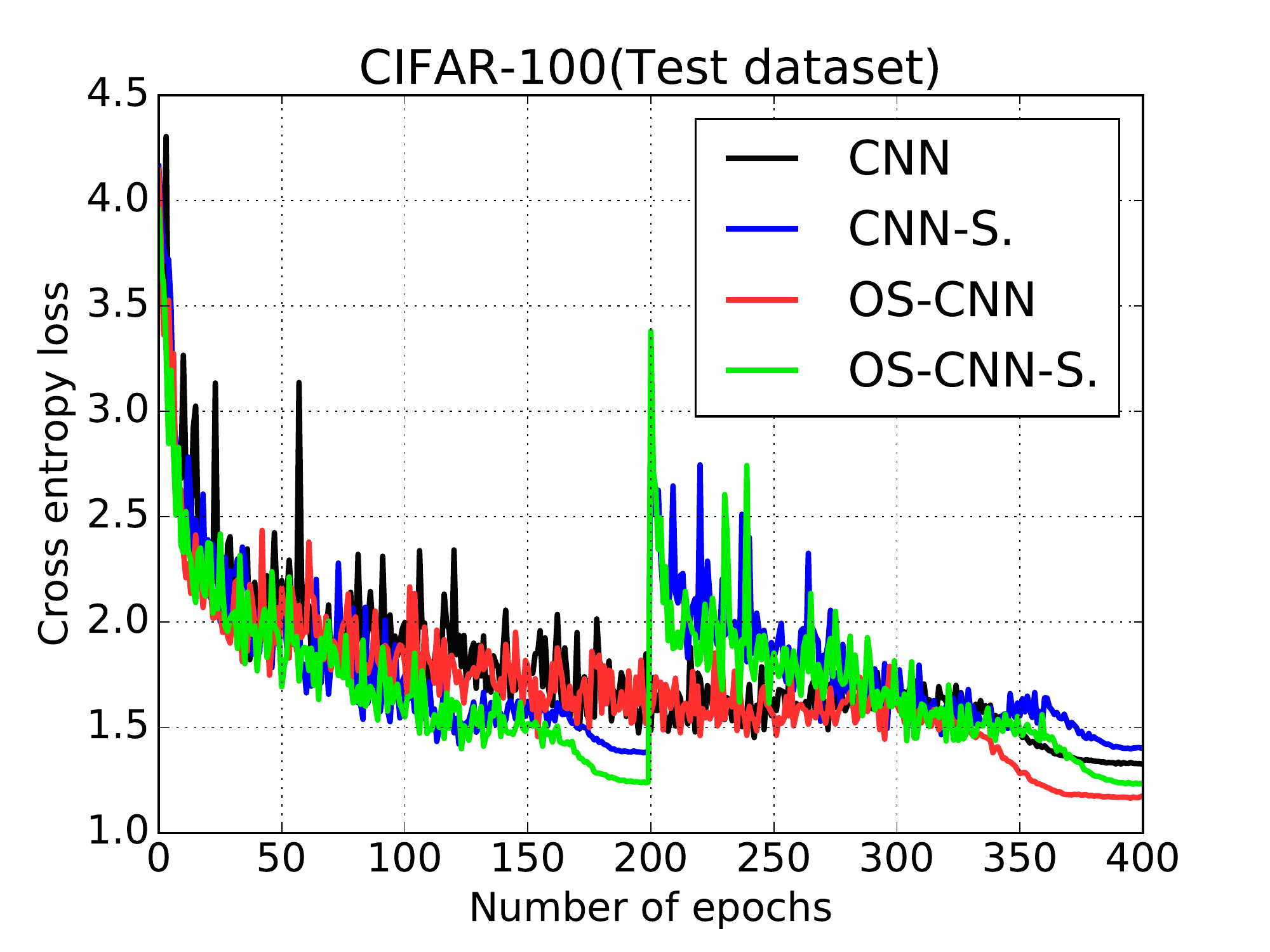}\label{20}
\end{center}
\caption{The cross-entropy loss and accuracy of CNN, CNN-SnapShot (CNN-S.), OS-CNN, and OS-CNN-SnapShot  (OS-CNN-S.) on the CIFAR-100 dataset. The left-hand column is for the training data, and the right-hand column is for the test data.}\label{loss}
\end{figure*}

\subsection{OSLNet on Large-sample Datasets} \label{sec:OSLNet on Large-sample Datasets}

To evaluate the applicability of our OSL on large-sample classification, we selected the following three datasets.

\begin{itemize}

\item CIFAR-10: This dataset~\cite{krizhevsky2009learning} consists of 60000 $32\times32$ color images in $10$ classes with 6000 images per class. There are $50000$ training images and $10000$ test images. The classes are as follows: airplane, automobile, bird, cat, deer, dog, frog, horse, ship and truck.

\item CIFAR-100: This dataset~\cite{krizhevsky2009learning} consists of 60000 $32\times32$ color images in $100$ classes with $600$ images per class. There are $500$ training images and $100$ test images per class. 

\item MNIST: This dataset of handwritten digits~\cite{lecun1998gradient} has a training set of $60000$ examples and a test set of $10000$ examples. It is a subset of a larger set available from MNIST. The digits have been size-normalized and centered in a fixed-size image. The size of the images is $28\times28$.
\end{itemize}
We compare the following methods: 1) CNN; 2) SnapShot ensembling of CNN (CNN-SnapShot); 3) a CNN with OSL (OS-CNN); and 4) SnapShot ensembling of OS-CNN (OS-CNN-SnapShot). The classification results are lised in Table~\ref{large-sample} and demonstrated in Fig.~\ref{loss}.

On the CIFAR-10 and CIFAR-100 datasets, for CNN, we used the VGG16 style network, where the convolutional layer has batch normalization, and the fully connected parts have $2$ hidden layers of $16$ units each. The epoch number is $400$, and the learning rate decreases from $0.01$ to $0$ by using the cosine annealed method. The optimization method iss SGD. For OS-CNN, we only replaced the last fully connected layer with the OS layer in CNN, while other settings unchanged. For CNN-SnapShot, the learning rate decreases from $0.01$ to $0$ by using the annealed cosine method within $200$ epochs, which restarted to the identical process one more time. The total number of epochs is $400$, and the number of SnapShot networks is $2$. Other settings are identical to those for CNN. For OS-CNN-SnapShot, except for the classification layer, other settings are identical to that for CNN-SnapShot.  

On the MNIST dataset, we followed the CNN structure published by PyTorch for MNIST. In the structure, there are two modules including convolution (Maxpooling and Relu activation) and two fully connected layers with 50 hidden neurons. The optimization method remains to be SGD, and the learning rate decreases from $0.01$ to $0$ as the cosine annealed method is used. The epoch number is $400$. For OS-CNN, CNN-SnapShot and OS-CNN-SnapShot, except for the structure of CNN, other settings were the same as that for the CIFAR-10 and CIFAR-100 datasets.

From the classification accuracy shown in Table~\ref{large-sample}, the curves of the cross-entropy loss, and the accuracy shown in  Fig.~\ref{loss}, it can be observed that CNN and OS-CNN have similar performance on the three datasets. That is, the proposed OSL is also applicable for large-sample classification.

\begin{table}[t]
\caption{Comparison of classification accuracy of CNN, CNN-SnapShot ensembling (CNN-S.), CNN with OSL (OS-CNN) and OS-CNN-SnapShop ensembing (OS-CNN-S.). } \label{large-sample}
\begin{center}
\begin{tabular}{c|cccc}\hline
\bf {Datasets}  &\bf {CNN} & \bf {CNN-S.} & \bf {OS-CNN} &\bf{OS-CNN-S.} 
\\ \hline
CIFAR-10       &0.9379   &    0.9445&   0.9418 &   0.9468 \\
\hline
CIFAR-100     & 0.7467 & 0.7648 & 0.7529  &0.7690 \\
\hline

MNIST          & 0.9921 & 0.9928  & 0.9925  & 0.9926    \\
          
\hline
\end{tabular}
\end{center} 
\end{table}

\subsection{Discussion} \label{sec:Discussion}
The focal loss places large weights on the samples that are difficult to identify; the center loss constructs a constraints that the features of samples from the same class must be close in the Euclidean distance; and the truncated $L_{q}$ loss is a noise-robust loss. In the experiments above, the three loss functions are not effective on the four small-sample datasets. It can also be observed that IterNorm and DBN perform unstably, which may be because, while accelerating optimization of neural networks, they do not have any particular mechanism to improve the classification performance of networks.

Among all the compared methods in the experiments on small-samples datasets, Dropout is an implicit ensemble method: in the training process it trains many subnetworks of the original network, and in the test phase no neuron is dropped. On the four small-sample datasets, Dropout do not show remarkable advantages over the network without dropouts.

Lsoftmax introduces a large margin into the cross-entropy softmax loss function to learn more discriminative features. The key problem of this method is that it is not easy to converge. Thus, it does not perform well in some experiments with a reduced training size. In contrast, the proposed OSL method converges easily. 

The SnapShot is a strong baseline among the compared methods and obtains relatively larger mean values and relatively smaller variances. However, although the method is trained only once in the training process and combines the predictions of all SnapShot networks in the test phase, due to the correlation of the base networks, the SnapShot increases the performance of FC slightly. 

The proposed method on small-samples datasets shows higher accuracy and better stability, which are mainly attributed to the OSL, where some connections are removed and the angles among the weights of different classes are maintained as $90^\circ$ from the beginning to the end during training. Having large angles among the weights of different classes in the classification layer is an important precondition to obtain a decision rule with a large margin, when the $L_2$ norms of the weights are not considered. In addition, which class each neuron of the last hidden layer should serve exclusively is determined prior to the start of training, so the difficulty of training can be significantly reduced. The experimental results of changing the width and depth of the network suggest that when a thin and shallow structure is selected, the OSL can obtain a better performance. In addition, experimental results of our OSL on the three large-sample image datasets show that, although developed for small-sample classification, OSL can also be well applied to large-sample classification.

Even though OSL has demonstrated superior performance, it has its own limitations. For example, due to the design of OSL, the number of neurons of the last hidden layer should be larger than the class number. Therefore, when the class number is large,~\emph{e.g.}, $1000$, a large number of neurons will be required, which will increase the flexibility of network. We leave it as an open problem for the future.

\section{Conclusions}\label{Conclusion}
In this paper, we proposed a new classification layer called ~\emph{Orthogonal Softmax Layer} (OSL). A network with OSL (OSLNet) has two advantages,~\emph{i.e.}, easy optimization and low Rademacher complexity. The Rademacher complexity of OSLNet is $\frac{1}{K}$, where $K$ is the number of classes, of that of a network with the fully connected classification layer. Experimental results on four small-sample datasets provide the following observations for an OSLNet in small-sample classification: 1) It is able to obtain higher accuracy with larger mean and smaller variance than those of the baselines used for comparison. 2) It is statistically significantly better than the baselines. 3) It is more suitable for thin and shallow networks than a fully connected network. Further experiments on three large-sample datasets show that, compared with a CNN with the fully connected classification layer, a CNN with the OSL has a competitive performance for both training and test data.

\bibliographystyle{IEEEtran}
\bibliography{main} 

\end{document}